\documentclass[10pt,journal,compsoc]{IEEEtran}

\usepackage[dvipsnames]{xcolor}
\usepackage{graphicx}

\usepackage{amsmath,amsthm,amssymb}
\usepackage{mathtools}
\usepackage{booktabs}
\usepackage{enumitem}
\usepackage{bm}
\usepackage{dblfloatfix}
\usepackage{url}

\DeclareMathOperator*{\argmax}{arg \ max}
\DeclareMathOperator*{\argmin}{arg \ min}

\newtheorem{proposition}{Proposition}

\begin{document}

\title{Higher-Order Explanations of Graph Neural Networks via Relevant Walks}

\author{Thomas Schnake, Oliver Eberle, Jonas Lederer, Shinichi Nakajima\\Kristof T. Sch\"utt, Klaus-Robert M\"uller, Gr\'egoire Montavon
\IEEEcompsocitemizethanks{%
\IEEEcompsocthanksitem T. Schnake, O. Eberle, J. Lederer, and K.T. Sch\"utt are with the Berlin Institute of Technology (TU Berlin), 10587 Berlin, Germany and BIFOLD -- Berlin  Institute  for  the Foundations  of  Learning  and  Data, Germany.

\IEEEcompsocthanksitem S. Nakajima is with the Berlin Institute of Technology (TU Berlin), 10587 Berlin, Germany,  BIFOLD -- Berlin  Institute  for  the Foundations  of  Learning  and  Data, Germany and RIKEN, Japan.

\IEEEcompsocthanksitem K.-R. M\"uller is with Google Research, Brain team, Berlin; the Berlin Institute of Technology (TU Berlin), 10587 Berlin, Germany; BIFOLD -- Berlin  Institute  for  the Foundations  of  Learning  and  Data, Germany; the Department of Artificial Intelligence, Korea University, Seoul 136-713, Korea; and the Max Planck Institut f{\"u}r Informatik, 66123 Saarbr{\"u}cken, Germany. E-mail: klaus-robert.mueller@tu-berlin.de.

\IEEEcompsocthanksitem G. Montavon is with the Berlin Institute of Technology (TU Berlin), 10587 Berlin, Germany and BIFOLD -- Berlin  Institute  for  the Foundations  of  Learning  and  Data, Germany. E-mail: gregoire.montavon@tu-berlin.de.
}
\thanks{(Corresponding Authors: Gr\'egoire Montavon, Klaus-Robert M\"uller)}}

\IEEEtitleabstractindextext{%
\begin{abstract}
Graph Neural Networks (GNNs) are a popular approach for predicting graph structured data. As GNNs tightly entangle the input graph into the neural network structure, common explainable AI approaches are not applicable. To a large extent, GNNs have remained black-boxes for the user so far.
In this paper, we show that GNNs can in fact be naturally explained using {\em higher-order} expansions, i.e.\ by identifying groups of edges that jointly contribute to the prediction.
Practically, we find that such explanations can be extracted using a nested attribution scheme, where existing techniques such as layer-wise relevance propagation (LRP) can be applied at each step. The output is a collection of walks into the input graph that are relevant for the prediction.
Our novel explanation method, which we denote by GNN-LRP, is applicable to a broad range of graph neural networks and lets us extract practically relevant insights on sentiment analysis of text data, structure-property relationships in quantum chemistry, and image classification.
\end{abstract}

}

\maketitle

\section{Introduction}

Many interesting structures found in scientific and industrial applications can be expressed as graphs. 
Examples are lattices in fluid modeling, molecular geometry, biological  interaction networks, or social\,/\,historical networks. 
Graph neural networks (GNNs) \cite{Scarselli:2009:GNN:1657477.1657482,Wu2020} have been proposed as a method to learn from observations in general graph structures and have found use in an ever growing number of applications \cite{schutt2018schnet,DBLP:journals/bioinformatics/ZitnikAL18,DBLP:conf/emnlp/MarcheggianiT17,DBLP:conf/emnlp/BastingsTAMS17,DBLP:conf/acl/CohnHB18,DBLP:journals/tog/WangSLSBS19}. 
While GNNs make useful predictions, they typically act as black-boxes, and it has neither been directly possible (1) to extract novel insight from the learned model nor (2) to verify that the model has made the intended use of the graph structure, e.g.\ that it has avoided Clever Hans phenomena \cite{lapuschkin2019unmasking}.

Explainable AI (XAI) is an emerging research area that aims to extract interpretable insights from trained ML models \cite{DBLP:series/lncs/11700}. So far, research has focused, for example, on full black-box models \cite{Ribeiro:2016:LIME,IntegratedGradient_SundararajanTY17}, or deep neural networks \cite{bach-plos15}, where in both cases, the prediction can be attributed to the input features. For a GNN, however, the graph being received as input is deeply entangled with the model itself, hence requiring a more sophisticated approach.

In this paper, we propose a theoretically founded XAI method for explaining GNN predictions. The conceptual starting point of our method is the observation that the function implemented by the GNN is locally polynomial with the input graph. This function can therefore be analyzed using a {\em higher-order} Taylor expansion to arrive at an attribution of the GNN prediction on collections of edges, e.g.\ {\em walks} into the input graph. ---Such an attribution scheme goes beyond existing XAI techniques for GNNs that are limited to identifying individual nodes or edges.

Furthermore, we find that the higher-order expansion can be expressed as a nesting of multiple first-order expansions, starting at the top layer of the GNN and moving towards the input layer. Specifically, we start with a node attribution task in the top layer and then grow these nodes into walks by recursively pursuing the expansion process w.r.t.\ quantities in the layers below. Practically, we can insert at every step any standard first-order explanation technique, in particular, the well-adopted Layer-wise Relevance Propagation (LRP) \cite{bach-plos15}. The resulting procedure that we propose and that we denote by GNN-LRP is shown in Figure \ref{fig:cartoon}.

\begin{figure}[h]
\centering
\includegraphics[width=\linewidth]{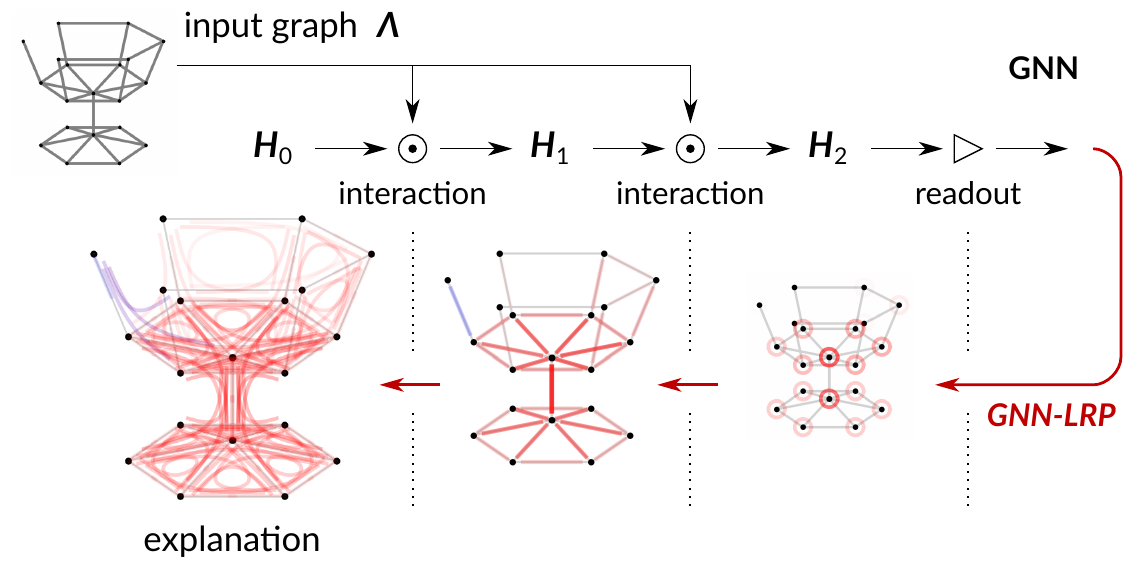}
\caption{High-level illustration of GNN-LRP. The explanation procedure starts at the GNN output, and proceeds backwards to progressively uncover the walks that are relevant for the prediction.}
\label{fig:cartoon}
\end{figure}

GNN-LRP applies directly to a broad range of GNN architectures, without need to learn a surrogate function, nor to run any optimization procedure. We demonstrate GNN-LRP on a variety of GNN models from diverse application fields: (1) a sentiment prediction model receiving sentence parse trees as input, (2) a state-of-the-art GNN for quantum mechanically accurate prediction of electronic properties from molecular graphs, and (3) a widely adopted image classifier that we view as a GNN operating on pixel lattices.---For each GNN model, our explanation method produces detailed and reliable explanations of the decision strategy from which novel application insights can be obtained.

\subsection{Related work}

We focus here on the related work that most directly connects to our novel GNN explanation approach, in particular, (1) explanation techniques based on higher-order analysis, and (2) explanation techniques that are specialized for GNNs. For a more comprehensive set of related works, we refer the reader to the review papers \cite{DBLP:journals/corr/Samek-XAI-review,DBLP:journals/inffus/ArrietaRSBTBGGM20} for XAI and \cite{Wu2020, DBLP:journals/corr/abs-1812-08434} for GNNs.

\subsubsection{Higher-Order Explanations}

Second-order methods (e.g.\ based on the model's Hessian) have been proposed to attribute predictions to pairs of input features \cite{Eberle2020,DBLP:journals/corr/Janizek2020, DBLP:conf/ecai/CuiMK20}. Another work \cite{DBLP:conf/kdd/CaruanaLGKSE15} incorporates an explicit sum-of-interactions structure into the model, in order to obtain second-order or higher-order explanations. Another approach \cite{DBLP:conf/iclr/TsangC018} detects higher-order feature interaction as an iterative algorithm which inspects neural network weights at the different layers. In the context of NLP, a special joint convolution-LSTM model was designed to extract n-ary relations between sentences \cite{zhang2018graph}.

Our work proposes instead to use the framework of Taylor expansions to arrive in a principled manner to the higher-order explanations, and it identifies GNNs as an important use case for such explanations.

\subsubsection{Explaining Graph Neural Networks}

The work \cite{DBLP:conf/cvpr/PopeKRMH19} extends explanation techniques such as Grad-CAM or Excitation Backprop to the GNN model, and arrives at an attribution on nodes of the graph. In an NLP context, graph convolutional networks (GCNs) have been explained in terms of nodes and edges in the input graph using the LRP explanation method \cite{pub10600}. The \mbox{`GNNExplainer'} \cite{DBLP:journals/corr/abs-1903-03894}, extracts the subgraph that maximizes the mutual information to the prediction for the original graph, and the identified subgraph gives the explanation. Further recently proposed methods that map the GNN prediction to graph substructures include XGNN \cite{DBLP:conf/kdd/YuanTHJ20}, Trap2 \cite{DBLP:journals/corr/abs-2004-09808} and GraphMask \cite{DBLP:journals/corr/abs-2010-00577}.

Our GNN-LRP method differs from these works by finding a scoring for {\em sequences} of edges (i.e.\ walks in the graph) instead of individual nodes or edges. This considerably enhances the informativeness of the produced explanations.

\section{Towards Explaining GNNs}
\label{section:towards}

Graph neural networks (GNNs) \cite{Scarselli:2009:GNN:1657477.1657482,Wu2020} are special types of neural networks that receive a graph as input. In practice, graphs can take a variety of forms, e.g.\ directed, undirected, labeled, unlabeled, spatial, time-evolving, etc. To handle the high heterogeneity of graph structures, many variants of GNNs have been developed (e.g.\ \cite{DBLP:conf/iclr/KipfW17, schnet_paper,DBLP:conf/ijcai/YuYZ18}). One commonality of most GNNs, however, is that the input graph is not located at the first layer, but occurs instead at multiple layers, by defining the connectivity of the network itself.

Graph neural networks are typically constructed by stacking several \textit{interaction blocks}. Each block $t = 1\dots T$ computes a graph representation $\bm{H}_t \in \mathbb{R}^{n \times d_t}$ where $n$ is the number of nodes in the input graph and $d_{t}$ is the number of dimensions used to represent each node. Within a block, the representation is produced by applying (i) an \textit{aggregate} step where each node receives information from the neighboring nodes, and (ii) a \textit{combine} step that extracts new features for each node. These two steps (cf.\ \cite{Scarselli:2009:GNN:1657477.1657482}) connect the representations $\bm{H}_{t-1}$ and $\bm{H}_t$ of consecutive blocks as:
\begin{align}
    \text{aggregate:} \quad \bm{Z}_t &= \bm{\Lambda} \bm{H}_{t-1} \label{eq:general_aggregate}\\
    \text{combine:} \quad \bm{H}_{t} &= \big(\mathcal{C}_{t} (\bm{Z}_{t,K}) \big)_{K} \label{eq:general_combine}
\end{align}
where $\bm{\Lambda}$ is the \textit{input graph} given as a matrix of size $n \times n$, e.g.\ the adjacency matrix to which we add self-connections. We denote by $\bm{Z}_{t,K}$ the row of $\bm{Z}_t$ associated to node $K$, and $\mathcal{C}_t$ is a `combine' function, typically a one-layer or multi-layer neural network, that produces the new representation for each node in the graph.

The whole input-output relation implemented by the GNN can then be expressed as a function
\begin{align}
    f(\bm{\Lambda};\bm{H}_0) = g\big(\bm{H}_T\big(\bm{\Lambda},\bm{H}_{T-1}\big(\bm{\Lambda}, \,\dots\, \bm{H}_1\big(\bm{\Lambda},\bm{H}_0\big)\big)\big)\big)
    \label{eq:general_readout}
\end{align}
which is a recursive application of Eqs.\ \eqref{eq:general_aggregate} and \eqref{eq:general_combine} starting from some \textit{initial state} $\bm{H}_0 \in \mathbb{R}^{n \times d_0}$, followed by a readout function $g$. The initial state typically incorporates information that is intrinsic to the nodes, or it can be set to constant values if no such information is present. The readout function is typically a classifier (or regressor) of the whole graph, but it can also be chosen to apply to subsets of nodes, for example, for node classification or link prediction tasks \cite{hu2020ogb, DBLP:journals/corr/abs-1812-08434}.

\subsection{First-Order Explanation}
\label{section:firstorder}

Consider first a `classical' approach to explanation where we attribute the output of the neural network to variables in the first layer. In the case of the GNN, the first layer is given by the initial state $\bm{H}_0$.

Let us now view the GNN as a function of the initial state, i.e.\ $f(\bm{H}_0)$. We will also denote by $\bm{H}_{0,I}$ the row of $\bm{H}_{0}$ associated to node $I$. A Taylor expansion of the function $f$ at some reference point $\widetilde{\bm{H}}_0$ gives:
\begin{align}
f(\bm{H}_0)  = \sum_I \bigg\langle \frac{\partial f}{\partial \bm{H}_{0,I}} \bigg|_{\widetilde{\bm{H}}_0} \!\!, (\bm{H}_{0,I} -\widetilde{\bm{H}}_{0,I}) \bigg\rangle + \dots
\label{eq:taylor}
\end{align}
where  `$\dots$' represents the zero-, second- and higher-order terms that have not been expanded, and where the sum represents the first-order terms. Because the sum runs over all nodes $I$ in the input graph, Eq. \eqref{eq:taylor} readily provides an attribution of the GNN output to these nodes.

It is arguable, however, whether this attribution can truly be interpreted as identifying node contributions. Indeed, the attribution may only reflect the importance of a node in the first layer, and not in the higher layers. Furthermore, an attribution of the prediction on nodes may not be sufficient for the application needs. For example, it does not tell us whether a node is important by itself, or if it is important because of its connections to other nodes or some more complex structure in the graph.

These limitations can be attributed to the fact that we have performed the decomposition w.r.t.\ the initial state $\bm{H}_0$ instead of the `true' input $\bm{\Lambda}$.

\subsection{Higher-Order Explanation}
\label{section:highorder}

Consider now the true input $\bm{\Lambda}$ of the GNN. Since it occurs at multiple layers and because layers are generally composed in a multiplicative manner (cf.\ Eqs.\ \eqref{eq:general_aggregate} and \eqref{eq:general_combine}), a first-order analysis of the GNN function $f(\bm{\Lambda})$ would not be suitable to identify the multiplicative interactions. These interactions can however be identified by applying a \textit{higher-order} Taylor expansion.

In the following, we will use the additional notation $\lambda_\mathcal{E}$ to denote the element of the matrix $\bm{\Lambda}$ associated to a particular edge $\mathcal{E}$ of the graph. Assuming that $f(\bm{\Lambda})$ is smooth on the relevant input domain, we can compute at some reference point $\widetilde{\bm{\Lambda}}$, a $T$-order Taylor expansion:
\begin{align}
f(\bm{\Lambda}) = \sum_{\mathcal{B}} \bigg[
& \, \frac{1}{\alpha_\mathcal{B}!}\frac{\partial^T f}{
\partial \lambda_{\mathcal{E}_1} \dots \partial \lambda_{\mathcal{E}_T}
}\bigg|_{\widetilde{\bm{\Lambda}}} \nonumber\\
& \hskip 1mm \cdot
 (\lambda_{\mathcal{E}_1} - \widetilde{\lambda}_{\mathcal{E}_1}) \cdot \hdots \cdot  (\lambda_{\mathcal{E}_T} - \widetilde{\lambda}_{\mathcal{E}_T})\bigg] \nonumber\\
&\hskip 3mm + \dots
\label{eq:hightaylor-general}
\end{align}
where we define $\alpha_{\mathcal{B}}! \vcentcolon= \prod_{ \mathcal{E}
} \alpha_{\mathcal{B},\mathcal{E}}!$ with $\alpha_{\mathcal{B},\mathcal{E}}$ denoting the number of occurrences of edge $\mathcal{E}$ in the bag $\mathcal{B}$. The sum runs over all bags $\mathcal{B}$ of $T$ edges. Hence, the terms of the sum capture the joint effect of multiple edges on the GNN output. The last line `$+ \dots$' represents the non-expanded terms of order lower or higher than $T$.

Note that with this formulation, one still faces the difficult task of specifying a meaningful reference point $\widetilde{\bm{\Lambda}}$ at which to perform the expansion and for which the Taylor expansion approximates the function well. In the following, we will leverage certain properties of the GNN model, to arrive at a simpler form of the Taylor expansion where the GNN output fully decomposes on bags of $T$ edges.

\begin{proposition}
\label{proposition:poshom}
Let $f(\bm{\Lambda})$ have the structure of Eqs.\ \eqref{eq:general_aggregate}--\eqref{eq:general_readout}, with $\mathcal{C}_t$ and $g$ piecewise linear and positively homogeneous with their respective inputs. If we perform a Taylor expansion of $f(\bm{\Lambda})$ as in \eqref{eq:hightaylor-general} at the reference point $\widetilde{\bm{\Lambda}}= s \bm{\Lambda}$ for an $s > 0$, then all terms of the expansion which are of higher or lower order than the network depth $T$ vanish in the limit of $s \to 0$, and we arrive at a decomposition $f(\bm{\Lambda}) = \sum_{\mathcal{B}} R_{\mathcal{B}}$ with
\begin{align}
R_\mathcal{B} &=  \frac{1}{\alpha_{\mathcal{B}}!}\frac{\partial^T f}{
\partial \lambda_{\mathcal{E}_1} \dots \partial \lambda_{\mathcal{E}_T}
} \cdot
 \lambda_{\mathcal{E}_1} \cdot \hdots \cdot \lambda_{\mathcal{E}_T}
\label{eq:hightaylor-reduced}
\end{align}
\end{proposition}

The proof of this proposition can be found in Appendix A in the Supplement. This attribution technique can also be seen as a higher-order generalization of 'Gradient $\times$ Input' \cite{DBLP:journals/corr/ShrikumarGSK16} and 'Hessian $\times$ Product' \cite{Eberle2020}.

\section{Explaining GNNs in Practice}

The higher-order Taylor expansion presented above is simple and mathematically founded. However, systematically extracting higher-order derivatives of a neural network is difficult and does not scale to complex models.

To address this first limitation, we introduce the concept of a \textit{walk} $\mathcal{W}$, which we define to be an {\em ordered} sequence of edges which connect nodes in consecutive layers of the GNN. The relation between bag-of-edges and walks is illustrated for a simple graph in \mbox{Fig.\ \ref{fig:bagwalk}}.
    
\begin{figure}[h]
    \centering
    \includegraphics[width=.97\linewidth]{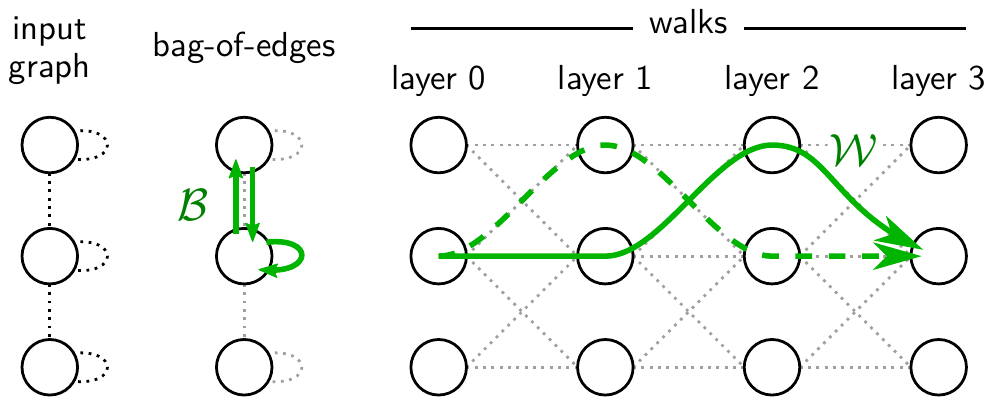}
    \caption{Illustration of a bag-of-edges $\mathcal{B}$ and the corresponding walks $\mathcal{W}$ for a simple input graph of three nodes. The two walks associated to the given bag-of-edges are shown with a solid and a dashed line, respectively.}
    \label{fig:bagwalk}
\end{figure}

Because each walk maps to a particular bag-of-edges, a walk-based explanation inherits all information contained in the bag-of-edges explanation. In particular, it is always possible to recover the bag-of-edges explanation from a walk-based explanation. However, using walks brings two further advantages: First, a walk-based explanation gives more information on way the multiple layers of the GNN have been used to arrive at the prediction. For example, as illustrated in Fig.\ \ref{fig:bagwalk}, it can reveal whether message passing between two nodes has occurred in the first or in the last layers of the GNN. Second, the fact that the walks more directly connect to the structure of the GNN confers important practical benefits for computing the explanations. In the following, we will contribute two algorithms for explanation, where the relevant walks can be easily extracted using simple backward passes in the GNN.

To simplify the presentation of these algorithms, we introduce the new variable $\bm{\Lambda}^\star \gets (\bm{\Lambda},\dots,\bm{\Lambda})$ which distinguishes between edges occurring at different layers of the GNN. We then express the GNN output as a function of this expanded input, i.e.\ $f(\bm{\Lambda}^\star)$. We also adopt a node-based notation, where walks are given by the sequence of nodes they traverse from the first layer to the top layer, e.g.\ $\mathcal{W} = (\dots,J,K,L,\dots)$. The letters $J,K,L$ denote nodes in consecutive layers, and `$\dots$' acts as a placeholder for the leading and trailing nodes of the walk. We further denote by $\lambda^\star_{JK}$ the element of $\bm{\Lambda}^\star$ representing the connection between node $J$ and node $K$.

\subsection{The \textit{GNN-GI} Baseline}
\label{section:gnngi}

Our first method, GNN-GI is based on the mathematical insight that Eq.\ \eqref{eq:hightaylor-reduced} can be rewritten in a way that decomposes into walks. Additionally, the $T$-order derivative reduces to simple first-order derivatives defined locally in the GNN, and that are easy to compute using backward passes.

\begin{proposition}
For the considered function $f(\bm{\Lambda}^\star)$ the higher-order terms in Eq.\ \eqref{eq:hightaylor-reduced}
can be equivalently computed as a sequence of differentations and multiplications by the terms of $\bm{\Lambda}^\star$ forming each walk $\mathcal{W}=(\dots, J,K,L, \dots)$:
\begin{align}
R_{\mathcal{W}} &= \frac{\partial}{\partial \hdots} \left( \frac{\partial}{\partial \lambda^\star_{JK}} \left( \frac{\partial \hdots}{\partial \lambda^\star_{KL}} \cdot  \lambda_{KL}^\star \right) \cdot \lambda_{JK}^\star \right) \cdot  \hdots
\label{eq:nestedgi}
\end{align}
and then applying the pooling operation $R_{\mathcal{B}} = \sum_{\mathcal{W} \in \mathcal{B}} R_{\mathcal{W}}$.
\label{proposition:nested}
\end{proposition}

A proof is given in Appendix B of the Supplement. Practically, GNN-GI operates as follows: Denote by $K,L,M$ the last three nodes of the walk $\mathcal{W}$, the algorithm starts by attributing on $\lambda^\star_{LM}$:
$$
R_{LM}
= \frac{\partial f}{\partial \lambda^\star_{LM} } \cdot \lambda_{LM}^\star = \Big\langle\frac{\partial f}{\partial \bm{H}_{T,M}}, \frac{\partial \bm{H}_{T,M}}{\partial \lambda^\star_{LM}}\Big\rangle \cdot \lambda_{LM}^\star,
$$
where $\bm{H}_{T,M}$ is the top-layer representation for node $M$. We observe that $R_{LM}$ is a function of $\bm{H}_{T-1,L}$, which itself depends on $\lambda^\star_{KL}$. Hence, the algorithm continues by taking $R_{LM}$ and attributing it on $\lambda^\star_{KL}$:
$$
R_{KLM} = \frac{\partial R_{LM}}{\partial \lambda^\star_{KL}} \cdot \lambda_{KL}^\star = \Big\langle\frac{\partial R_{LM}}{\partial \bm{H}_{T-1,L}}, \frac{\partial \bm{H}_{T-1,L}}{\partial \lambda^\star_{KL}}\Big\rangle \cdot \lambda_{KL}^\star.
$$
To avoid writing all indices along the walk, we use the notation $R_{KL\dots}$ in place of $R_{KLM}$, and then proceed further to compute $R_{JKL\dots}$. We continue along the walk $\mathcal{W}$ until we reach the first layer where the final score $R_\mathcal{W}$ is obtained. The procedure can be repeated for all walks in the graph to arrive at a complete explanation.---We call this algorithm GNN-GI, because it can interpreted as a nested Gradient$\,\times\,$Input (GI) attribution, and consequently brings GI to GNNs. We will use GNN-GI as a baseline method in \mbox{Section \ref{section:validating}}.

GNN-GI has however several limitations in terms of explanation quality: When the network is not positive homogeneous (e.g.\ because its neurons have non-zero biases), the conservation property $\sum_\mathcal{B} R_\mathcal{B} = f(\bm{\Lambda})$ is no longer satisfied. A significant portion of the prediction might consequently fail to be attributed to the input variables. Furthermore, in deep models, the gradient on which the GI method relies is affected by a shattering effect \cite{DBLP:conf/icml/BalduzziFLLMM17,DBLP:journals/dsp/MontavonSM18} which makes it noisy and less reliable. More sophisticated techniques are therefore needed.

\subsection{Better Explanations with \textit{GNN-LRP}}
\label{section:gnnlrp}

To overcome the limitations of our simple GNN-GI baseline, we will leverage the robustness and broader applicability of an existing explanation technique, Layer-wise Relevance Propagation (LRP) \cite{bach-plos15,DBLP:series/lncs/MontavonBLSM19}, and adapt it for the purpose of explaining a GNN.

LRP is a well-adopted explanation technique that is specifically designed for deep neural networks (DNNs). LRP propagates the DNN output from the top layer to the input layer. At each layer, a propagation rule is applied. The propagation rule has hyperparameters that can be set to induce certain desirable properties of the explanation, such as robustness. Compared to GI, LRP deals better with the strong nonlinearities of the model (e.g.\ the shattering effect \cite{DBLP:conf/icml/BalduzziFLLMM17}). GI can in fact be seen as a special case of LRP where the hyperparameters for robustness are set to zero \cite{DBLP:journals/corr/ShrikumarGSK16,DBLP:series/lncs/Montavon19}. Furthermore, when the neural network has negative biases, e.g.\ to build sparse representations, LRP also conserves the prediction better than GI \cite{DBLP:series/lncs/Montavon19}.

Motivated by the advantageous properties of LRP, we propose GNN-LRP: an improvement of GNN-GI that substitutes in Eq.\ \eqref{eq:nestedgi} the GI attribution steps by LRP attribution steps. The new nested attribution procedure is given by:
\begin{align}
R_\mathcal{W} = \text{LRP}\Big(\underbrace{\text{LRP}\Big(\underbrace{\text{LRP}\Big(\dots,\bm{\lambda}^\star_{KL}\Big)}_{\bm{R}_{KL\dots}},\bm{\lambda}^\star_{JK}\Big)}_{\bm{R}_{JKL\dots}},\dots \Big),
\label{eq:nestedlrp}
\end{align}
where the function LRP$(\cdot,\cdot)$ attributes the relevance given as a first argument to the connections given as a second argument. Note that LRP works at the neuron level instead of the node level. Hence, we use the bold notation in \mbox{Eq.\ \eqref{eq:nestedlrp}} to signify that the relevance arriving at a given node is a collection of relevances arriving at each neuron of the node, and that a connection between two nodes is a collection of connections between pairs of neurons from the two nodes. The nested LRP procedure is illustrated in Fig.\ \ref{fig:gnnlrp}.

\begin{figure}[h]
    \centering
    \includegraphics[width=.95\linewidth]{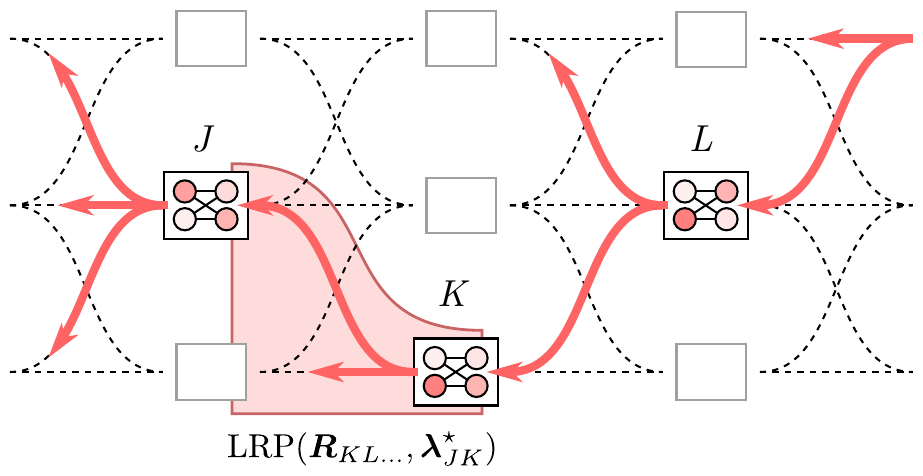}\vskip -2mm
    \caption{Illustration of GNN-LRP, showing the sequence of LRP computations needed for determining the relevance score $R_\mathcal{W}$ associated to a particular walk $\mathcal{W} = (\dots,J,K,L,\dots)$ in the input graph.}
    \label{fig:gnnlrp}
\end{figure}

\begin{table*}[t!]
 \caption{Practical GNN-LRP propagation rules for different types of GNNs. We designate by $w_{jk}$ the element of the matrix $\bm{W}_t$ that links neuron $j$ to neuron $k$. The function $\rho(\cdot)$ is the rectification function $\max(0,\cdot)$, and we make use of the notation $(\cdot)^\wedge = (\cdot) + \gamma \rho(\cdot)$, where $\gamma$ is a hyperparameter of the propagation rule. In the second row, $\delta_{jk}$ is an indicator function (Dirac delta) that is $1$ when $j$ and $k$ are neurons at the same location within their respective node. In the last row, $\bm{\Lambda}_s$ represents one component of the graph convolutional filter approximation presented in \cite{spectral_nets_bruna2014}.}
\label{table:propagationrules}
\vspace{-5mm}
\small
\begin{align}
\toprule
& \text{Model} & ~ & \text{Aggregate} & ~ & \text{Combine} & ~ & \text{GNN-LRP Rule} \nonumber\\
\midrule
& \text{GCN \cite{DBLP:conf/iclr/KipfW17}} &
\bm{Z}_t &= \bm{\Lambda} \bm{H}_{t-1} & 
\bm{H}_t &= \rho(\bm{Z}_t \bm{W}_t)
 &  R_{jKL\dots} &= \sum_{k \in K}\frac{\lambda_{JK} h_{j} w_{jk}^\wedge}{\sum_J \sum_{j \in J} \lambda_{JK} h_{j} w_{jk}^\wedge } R_{kL\dots} \label{eq:lrp-gcn}
 \\\midrule
 & \text{GIN \cite{DBLP:conf/iclr/XuHLJ19}} &
 \bm{Z}_t &= \bm{\Lambda} \bm{H}_{t-1} & 
\bm{H}_t &= (\text{MLP}^{(t)} (\bm{Z}_{t,K}))_{K} &
 R_{jKL\dots} &= \sum_{k \in K} \frac{\lambda_{JK} h_{j} \delta_{jk}}{\sum_J \sum_{j \in J} \lambda_{JK} h_{j} \delta_{jk}} \text{LRP}(R_{KL\dots},z_k) \label{eq:lrp-gin}
 \\\midrule
 & \text{Spectral \cite{spectral_nets_bruna2014,NIPS2016_6081}} &
\bm{Z}_{s,t} &= \textstyle \bm{\Lambda}_{s} \bm{H}_{t-1} & 
\bm{H}_t &= \textstyle \rho(\sum_s \bm{Z}_{s,t} \bm{W}_{s,t}) & 
 R_{jKL\dots} &= \sum_{k \in K}\frac{\sum_s  h_{j} (\lambda_{JK}^sw_{jk}^s)^\wedge}{\sum_J \sum_{j \in J} \sum_s h_{j} (\lambda_{JK}^sw_{jk}^s)^\wedge} R_{kL\dots} \label{eq:lrp-spectral}
 \\\bottomrule\nonumber
\end{align}\vskip -7mm
\end{table*}

While GNN-LRP does not correspond to a particular analytical form, the approach can still be theoretically justified using the same mathematical tool that was used to justify LRP: Deep Taylor Decomposition (DTD) \cite{DBLP:journals/pr/MontavonLBSM17}. DTD views LRP as performing a multitude of simple Taylor expansions in each layer of the deep network---one per neuron. DTD relies on an inductive principle that we can formulate in the context of GNNs in three steps. Denote by $R_{kL\dots}$ the relevance arriving in neuron $k$, and let $h_k$ be the corresponding neuron activation.

\medskip

\begin{description}
\item[Step 1] Assume that we can write $R_{kL\dots} = h_k c_{kL\dots}$ where $c_{kL\dots}$ is a locally approximately constant term.
\item[Step 2] Define the map $(\bm{\lambda}^\star_{Jk})_J \mapsto h_k c_{kL\dots}$, where $c_{kL\dots}$ is treated as constant, and apply LRP to produce the relevance scores $R_{jkL\dots}$.
\item[Step 3] Apply the pooling $R_{jKL\dots} = \sum_{k \in K} R_{jkL\dots}$ and verify that the result can be written as $h_j c_{jKL\dots}$ with $c_{jKL\dots}$ locally approximately constant.
\end{description}

\medskip

\noindent If this induction holds, this justifies the application of LRP through the whole GNN from the top-layer to the first layer. Appendix C of the Supplement gives a concrete application of these induction steps for the GCN model. Practical GNN-LRP propagation rules for three popular GNN architectures are provided in Table \ref{table:propagationrules}.

Our GNN-LRP method is applicable to a broad range of existing GNN models, in particular, the original GNN model \cite{Scarselli:2009:GNN:1657477.1657482}, \textit{GCN} \cite{DBLP:conf/iclr/KipfW17}, \textit{GraphSAGE} with mean aggregation \cite{hamilton2017inductive}, \textit{Neural FP} \cite{NIPS2015_5954}, \textit{GIN} \cite{DBLP:conf/iclr/XuHLJ19}, any spectral filtering method \cite{DBLP:journals/spm/BronsteinBLSV17} such as the \textit{Spectral Network} \cite{spectral_nets_bruna2014} or \textit{ChebNet} \cite{NIPS2016_6081}. GNN-LRP is also applicable to other recent GNN architectures such as the \textit{SchNet} \cite{schnet_paper} used for predicting molecular properties, and where the graph is a representation of the distance between atoms. GNN-LRP is also applicable to convolutional neural networks for computer vision such as \textit{VGG-16} \cite{DBLP:journals/corr/SimonyanZ14a}, which can been seen as a particular GNN receiving as input a pixel lattice. Because GNN-LRP is tightly embedded in the more general LRP/DTD framework, it can be extended to more advanced architectures such as joint CNN-GNN models for spatio-temporal graphs \cite{DBLP:conf/ijcai/YuYZ18}. Furthermore, GNN-LRP allows in principle to use different edges and nodes at different layers, which can be useful to handle graph pooling structures, such as those described in \cite{DBLP:conf/nips/YingY0RHL18}.

\subsubsection{Implementing GNN-LRP}
\label{section:implementation}

While equations in Table \ref{table:propagationrules} fully specify the GNN-LRP procedure, the latter can be more conveniently and concisely implemented using a `forward-rewrite' strategy originally proposed for standard LRP (cf.\ \cite{DBLP:journals/corr/Samek-XAI-review}). In the following, we adapt this strategy to GNN-LRP.

Our demonstration focuses on the GCN case (Eq.\ \eqref{eq:lrp-gcn}), but the approach can be extended to other GNN models. We first observe that by rewriting relevance scores as a product of the corresponding activation and some factor, i.e.\ $R_{jKL\dots} = h_{j} c_{jKL\dots}$ and $R_{kL\dots} = h_{k} c_{kL\dots}$, Eq.\ \eqref{eq:lrp-gcn} can be equivalently formulated as:
\begin{align}
c_{jKL\dots} = \sum_{k \in K}\lambda_{JK} w_{jk}^\wedge  \frac{h_k}{p_k} c_{kL\dots}
\label{eq:lrp-alternate}
\end{align}
with $p_k = \sum_{J}\sum_{j \in J} \lambda_{JK} h_j w_{jk}^\wedge$. This formulation can be further transformed to let appear a structure similar (but not equivalent) to gradient propagation:
\begin{align}
c_{jKL\dots} &= \sum_{k \in K} \frac{\partial p_k}{\partial h_j}  \frac{h_k}{p_k} c_{kL\dots}
\label{eq:lrp-gradient}
\end{align}
Specifically, compared to the standard gradient propagation equation given by $\delta_j = \sum_{K} \sum_{k \in K} (\partial h_k / \partial h_j) \delta_k$, our LRP-equivalent equation differs by (i) computing the derivative of $p_k$ instead of $h_k$, (ii) multiplying by an extra factor $h_k/p_k$, and (iii) not summing over nodes $K$.

Having connected GNN-LRP to gradient propagation, we now rewrite the forward pass of the GNN in a way that its output remains the same, but where the output of the gradient computation---assumed to be computed by automatic differentiation---matches Eq.\ \eqref{eq:lrp-gradient}. Such functionality can be achieved by `detaching' certain parts of the forward computation from the differentiation graph. Specifically, we redefine the combine step of Eq.\ \eqref{eq:lrp-gcn} as:
\begin{align*} 
\bm{P}_t &\gets \bm{Z}_t \bm{W}_t^\wedge\\
\bm{Q}_t &\gets \bm{P}_t \odot [\kern1pt \rho(\bm{Z}_t \bm{W}_t) \oslash  \bm{P}_t ]_\text{cst.}\\[1mm]
\bm{H}_t &\gets \bm{Q}_t \odot \bm{M}_K + [\bm{Q}_t]_\text{cst.} \odot (\bm{1}-\bm{M}_K)
\end{align*}
where $\odot$ and $\oslash$ denote the element-wise multiplication and division respectively, where $\bm{M}_K$ is a mask array that is one for neurons associated to node $K$ and zero elsewhere, and where $[\cdot]_\text{cst.}$ indicates portions of the computation that have been detached.

Furthermore, once the forward pass is rewritten at each layer, with masks chosen in a way that it selects for some walk $\mathcal{W}$ of interest, we can compute the relevance score for $\mathcal{W}$ as
$$
R_\mathcal{W}  = \Big\langle \textsc{Autograd}( f, \bm{H}_{0,I}) \,,\, \bm{H}_{0,I}\Big\rangle.
$$
where the automatic differentiation mechanism is now fully repurposed to perform the sequence of LRP computations. The procedure must then be repeated for every walk in the graph in order to produce a full explanation.

Note that this procedure is flexible and can support many variations. For example, if we wish to ignore the node at which the walk passes at a certain layer of the GNN, the mask at that layer can be simply removed. In effect, removing this mask is equivalent to summing the relevance score of walks passing through all nodes at that layer.

Also, because computing a full explanation requires as many forward/backward steps as there are walks in the input graph, various strategies can be considered to speed up computations: For example, we can adopt a multi-resolution approach where relevant walks are first defined at a coarse level between super-nodes (masks $\bm{M}$ are then set to match these super-nodes), and the most relevant `super-walks' are then  expanded into walks between actual nodes. Alternately, the GNN-LRP computation can be broken down layer-wise, and walk computations can be stopped early if the relevance score $R_{JKL \dots}$ is below a certain threshold. Finally, GNN-LRP can also be made faster when the input graphs have a particular structure. In particular, for the experiments of Section \ref{section:vgg}, we will leverage the local connectivity of the VGG-16 network to compute many walks in parallel, which dramatically speeds up computations. Details of this parallel walk computation approach are presented in Appendix D of the Supplement.

\section{Testing and Validating GNN-LRP}
\label{section:validating}

To test the proposed GNN-LRP method, we train various types of GNNs on several graph prediction tasks. We first consider a two-class synthetic problem where the first class consists of Barab\'asi-Albert graphs \cite{Albert2002} of growth parameter $1$ (BA-1), and where the second class has a slightly higher growth model and new nodes attached preferably to low-degree nodes. Examples of graphs from the two classes are given in Fig.\ \ref{fig:gcn-classes} (left). Details on this synthetic dataset are given in Appendix E in the Supplement.

We consider first a graph isomorphism network (GIN) that we train on this task. Our GIN has two interaction layers. In each interaction layer, the combine function consists of a two-layer network with $32$ neurons per node at each layer. The initial state $\bm{H}_0$ is an all-ones matrix of size $n \times 1$, i.e.\ nodes do not have intrinsic information. The GIN receives as input the connectivity matrix $\bm{\Lambda} = \widetilde{\bm{A}}/2$ where $\widetilde{\bm{A}}$ is the adjacency matrix augmented with self-connections. The GIN is trained on this task until convergence, where it reaches an accuracy above $95\%$. After training, we take some exemplary input graph from the class BA-1, predict it with our GIN, and apply GNN-LRP on the prediction. We use the LRP parameter $\gamma=2$ and $\gamma=1$ in the first and second interaction layers respectively. The resulting explanation is shown in Fig.\ \ref{fig:gcn-classes} (right).

\begin{figure}[h!]
    \small \centering \sffamily \footnotesize
    \parbox{.17\linewidth}{\centering Class 1}%
    \parbox{.17\linewidth}{\centering Class 2}%
    \parbox{.66\linewidth}{\centering Explanation for an example of class 1}\\[2mm]
    \vline\hfill
    \includegraphics[height=.65\linewidth]{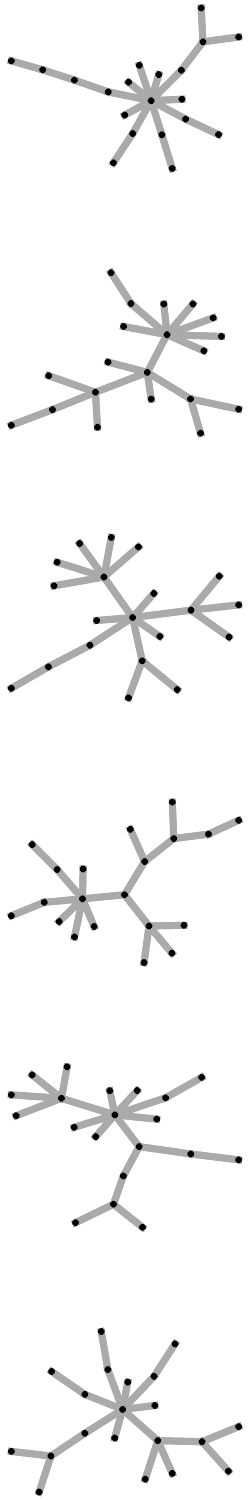}
    \hfill\vline\hfill
    \includegraphics[height=.65\linewidth]{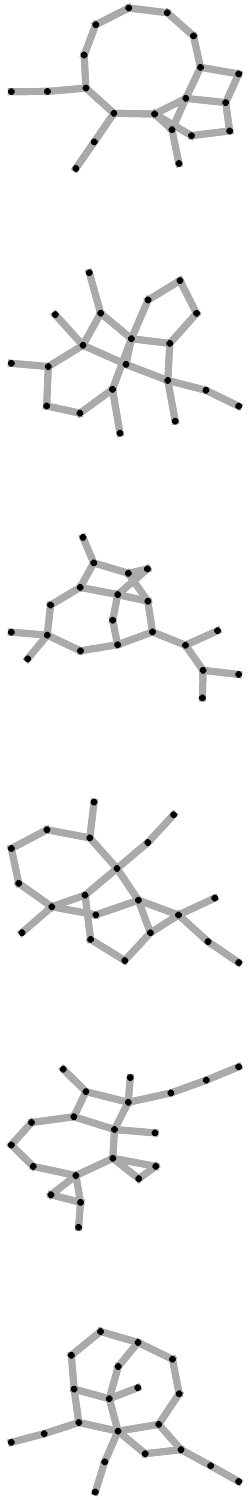}
    \hfill\vline\hfill
    \includegraphics[height=.65\linewidth,clip=True,trim=10 10 10 10]{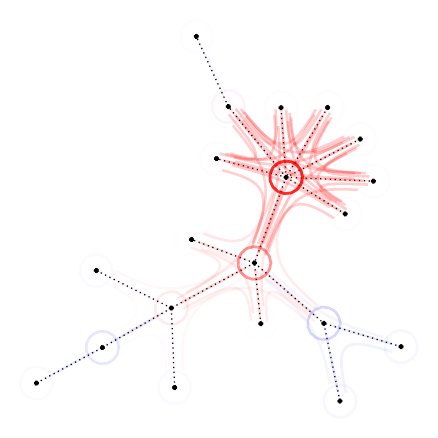}
    
    \caption{{\em Left:} Examples from the two classes of our synthetic dataset. {\em Right:} GNN-LRP explanation for the GIN prediction of a graph of class 1. Relevant (positively contributing) walks are shown in red and negatively contributing walks are in blue. Circles represent walks (or part of the walks) that are stationary.}
    \label{fig:gcn-classes}
\end{figure}

The explanation produced by GNN-LRP reveals that walks that traverse or stay in the high-degree node are the principal contributors to the GIN prediction. On the other hand leaf nodes or sequences of low-degree nodes are found to be either irrelevant or to be in slight contradiction with prediction.

In the following, we compare GNN-LRP to a collection of other GNN explanation methods:
\begin{enumerate}[label=(\alph*)]
    \item {\em Pope et al.} \cite{DBLP:conf/cvpr/PopeKRMH19}: The method views the GNN as a function of the initial state $\bm{H}_0$, and performs an attribution of the GNN output on nodes as represented in $\bm{H}_0$. In principle, the proposed framework lets the user choose the technique to perform attribution on $\bm{H}_0$. In our benchmark, we use the techniques GI and LRP.
    \item {\em GNN-GI}: This is the simple baseline we have contributed in Section \ref{section:gnngi}. This baseline can also be seen as a special case of GNN-LRP with $\gamma=0$.
    \item {\em GNNExplainer} \cite{DBLP:journals/corr/abs-1903-03894}: The method runs an optimization problem that finds a selection of edges that maximize the model output. The procedure can be viewed as finding a mask $\bm{M} = \sigma(\bm{R})$ where $\sigma$ denotes the logistic sigmoid function, that maximize the prediction $f(\bm{M} \odot \bm{\Lambda})$. The explanation is then given by $\bm{R}$.
\end{enumerate}
Explanations produced by each method are shown in Fig.\ \ref{fig:compare-benchmark}.
\begin{figure}[t]
    \footnotesize \sffamily
    \centering
    \parbox{.38\linewidth}{\centering Pope et al.\ \cite{DBLP:conf/cvpr/PopeKRMH19}\\(GI / LRP)}\hfill
    \parbox{.18\linewidth}{\centering GNN-GI}\hfill
    \parbox{.18\linewidth}{\centering \textbf{GNN-LRP}}\hfill
    \parbox{.18\linewidth}{\centering GNNExpl- ainer \cite{DBLP:journals/corr/abs-1903-03894}}\\[1mm]
    \includegraphics[width=.18\linewidth,clip=True,trim=50 0 38 0]{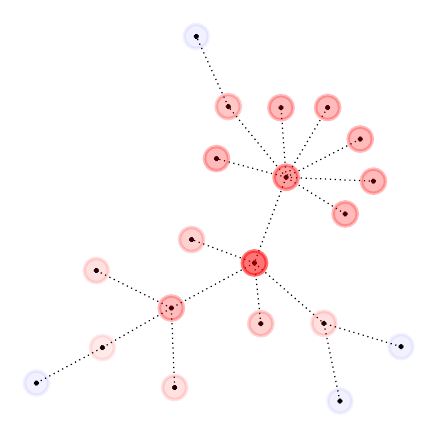}\hfill\vline\hfill
    \includegraphics[width=.18\linewidth,clip=True,trim=50 0 38 0]{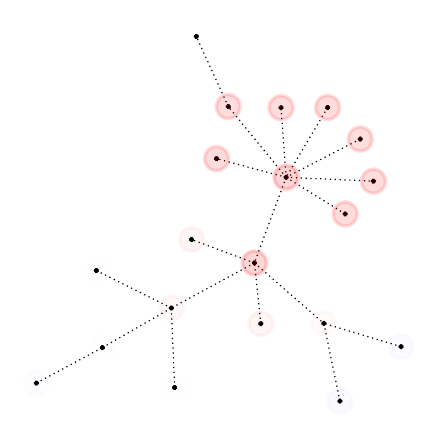}\hfill\vline\hfill
    \includegraphics[width=.18\linewidth,clip=True,trim=50 0 38 0]{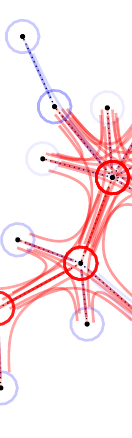}\hfill\vline\hfill
    \includegraphics[width=.18\linewidth,clip=True,trim=50 0 38 0]{figures/BA/gnnlrp.pdf}\hfill\vline\hfill
    \includegraphics[width=.18\linewidth,clip=True,trim=50 0 38 0]{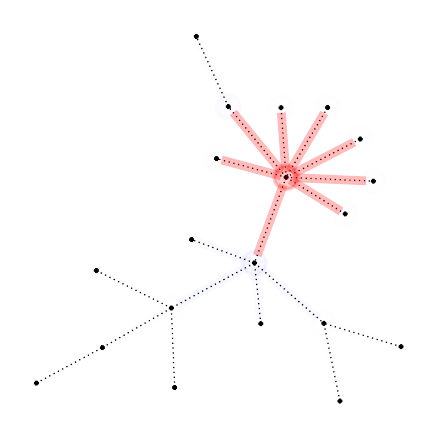}
    \caption{Comparison of different explanation techniques on the same graph as in Fig.\ \ref{fig:gcn-classes}. GNN-LRP produces more detailed explanations compared to \cite{DBLP:conf/cvpr/PopeKRMH19} and \cite{DBLP:journals/corr/abs-1903-03894}, and brings more robustness compared to GNN-GI.}
    \label{fig:compare-benchmark}
\end{figure}
The method by Pope et al. \cite{DBLP:conf/cvpr/PopeKRMH19} highlights nodes that are relevant for the prediction. However, it is difficult to determine from the explanation whether the highlighted nodes are relevant by themselves or if they are relevant in relation to their neighbors. The GNN-GI baseline we have contributed, and our more advanced method GNN-LRP, provide a much higher level of granularity: They distinguish in particular between what has to be attributed to the node itself and what has to be attributed to collections of multiple connected nodes. The GNNExplainer \cite{DBLP:journals/corr/abs-1903-03894} produces explanations that are in agreement with GNN-LRP but sparser and less detailed.

Even when the user aims specifically for a simple explanation, GNN-LRP remains a method of choice: Relevance scores $(R_\mathcal{W})_\mathcal{W}$ can be reassigned from walks to more descriptive structures, such as leaves, trees, cycles, etc. Relevant walks can also be reassigned to more basic elements such as nodes or edges in order to produce a level of complexity similar to Pope et al.\ \cite{DBLP:conf/cvpr/PopeKRMH19} or GNNExplainer \cite{DBLP:journals/corr/abs-1903-03894}. In fact, pooling GNN-LRP and GNN-GI explanations based on the first node of the walk results in node-based explanations that are equivalent to those obtained with Pope et al.\ \cite{DBLP:conf/cvpr/PopeKRMH19}.

Orthogonal to the level of detail of the explanation, methods based on GI tend to be less selective than those based on LRP. We observe for GNN-GI that several regions of the graph receive an excessive amount of positive or negative relevance. The noise caused by GI becomes particularly strong and problematic on larger and deeper models such as the VGG-16 network (cf.\ Fig.\ \ref{fig:vgg_walks} in Section \ref{section:vgg}).

Overall, GNN-LRP is the only method in our benchmark that produces explanations that have both the desired robustness and a high level of detail.

\subsection{Quantitative Comparison}
\label{section:quantitative}

To verify that GNN-LRP produces explanations that are \textit{systematically} superior to those produced by other explanation methods, we will perform a quantitative evaluation. A common evaluation technique that was introduced in the context of image classifiers is pixel-flipping \cite{DBLP:journals/tnn/SamekBMLM17}. The procedure consists of `flipping' pixels by increasing or decreasing order of relevance to verify that removing relevant/irrelevant pixels causes a strong/weak variation at the output of the model. If that is the case, it can be concluded that the explanation faithfully describes the prediction at the output of the model.

We adapt the pixel-flipping method to graph data by considering {\em nodes} of the graph instead of pixels of an image as a flipping unit. We refer to this adaptation as `node-flipping'. In our benchmark, we will leverage the higher-order information contained in some of the explanations to determine more precisely which nodes need to be flipped to incur the desired effect on the GNN output. Specifically, based on the explanation, we wish to estimate for any subgraph $\mathcal{G}$ the contribution $R_\mathcal{G}$ of that subgraph to the prediction, so that the effect on the GNN output of adding/removing nodes can be expressed as differences $R_\mathcal{G} - R_\mathcal{G'}$.

For a node-based attribution (e.g.\ \cite{DBLP:conf/cvpr/PopeKRMH19}), the relevance to be attributed to a particular subgraph $\mathcal{G}$ is simply given as the sum of contributions of each node:
$$
R_{\mathcal{G}} = \textstyle \sum_{\mathcal{V}\in \mathcal{G}} R_\mathcal{V}
$$
For an edge-based attribution (e.g. GNNExplainer \cite{DBLP:journals/corr/abs-1903-03894}), the relevance of the subgraph is more finely given as a sum over the edges $\mathcal{E}$ forming the subgraph:
$$
R_{\mathcal{G}} = \textstyle \sum_{\mathcal{E} \in \mathcal{G}} R_\mathcal{E}
$$
For higher-order attribution methods whose unit of explanation are bags of edges (computable as $R_\mathcal{B} = \sum_{\mathcal{W} \in \mathcal{B}} R_\mathcal{W}$ for GNN-LRP and GNN-GI), the estimation of the subgraph relevance is further refined by computing
\begin{align*}
R_{\mathcal{G}} &= \textstyle \sum_\mathcal{B \in \mathcal{G}} R_\mathcal{B}
\end{align*}
where the membership relation $\mathcal{B} \in \mathcal{G}$ requires that all edges in the bag $\mathcal{B}$ are also in the subgraph $\mathcal{G}$.
We note that in every case, the estimator $R_{\mathcal{G}}$ is zero when considering the empty graph, and it corresponds instead to the GNN output when considering the full graph.

We consider two node-flipping tasks. First, an \textit{activation task}, where we would like to add nodes to the graph so as to activate the GNN output maximally and as quickly as possible. In the general case, finding the optimal sequence of nodes to add is combinatorially complex. However we can use instead a \textit{greedy} approach, where for a given subgraph $\mathcal{G}$, the best node to add next according to the explanation is given by:
 \begin{align}
 \mathcal{V}^\star = \argmax_{\mathcal{V} \notin \mathcal{G}} \, R_{\mathcal{G} \cup \{ \mathcal{V} \}} \label{eq:flipping-activation}
 \end{align}
Then, we consider a \textit{pruning task}, where we start with the original graph $\mathcal{G}_0$ and want to remove nodes from the graph in a way that the GNN output is minimally affected. In that case, the best node to remove is
 \begin{align}
 \mathcal{V}^\star = \argmin_{\mathcal{V} \in \mathcal{G}} \, \big| {R_{\mathcal{G}_0}} - R_{\mathcal{G} / \{\mathcal{V}\}}\big|.
\label{eq:flipping-pruning}
\end{align}
Note that with this greedy strategy, higher-order techniques such as GNN-GI and GNN-LRP may have their performance underestimated, as adding one node at a time prevents from uncovering new parts of the graph composed of mutually relevant nodes. To mitigate this effect we can substitute at regular intervals of the node-flipping procedure the greedy step by a coarser step where the node is added based on its marginal score $R_\mathcal{V}$.

Once the sequence of nodes to flip has been built, we produce a plot that keeps track of the GNN output throughout the flipping process. The procedure can be repeated for a large number of graphs, leading to an averaged version of the same plot. The latter can be further summarized with the `Area Under the Flipping Curve' (AUFC), which we define as the mean of plotted values over the flipping sequence. The node-flipping procedure and the AUFC computation are depicted in Fig.\ \ref{fig:pflip}.

\begin{figure}[h]
    \centering
    \includegraphics[width=1.0\linewidth]{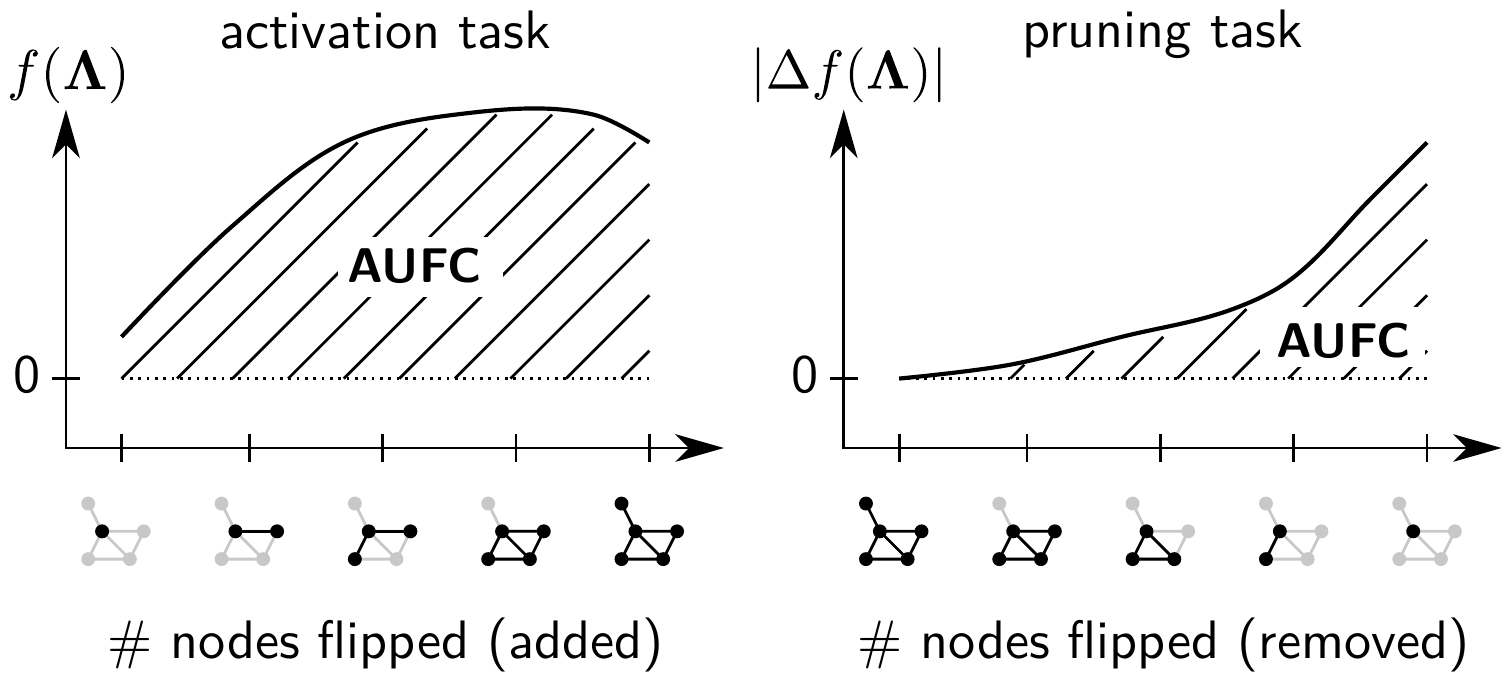}
    \caption{Depiction on the activation and pruning tasks, of the `node-flipping' procedure and the AUFC computation. We denote by $|\Delta f(\bm{\Lambda} )|$ the quantity $|f(\bm{\Lambda}) - f(\bm{\Lambda}_0) |$ where $\bm{\Lambda}_0$ is the input graph before pruning.}
    \label{fig:pflip}
\end{figure}

We will now use the node-flipping procedure to evaluate the different explanation methods on a broad set of architectures and tasks. On the {\em Synthetic Data}, we consider in addition to the GIN model used in the qualitative experiments, a GCN and a spectral network. Our GCN, has $128$ neurons per node at each layer. For the spectral network, we use $32$ neurons per node at each layer and give as input the power expansion $\bm{\Lambda} = [\widetilde{\bm{A}}^0,\frac12 \widetilde{\bm{A}}^1,\frac14 \widetilde{\bm{A}}^2]$. Both networks perform similarly to the GIN in terms of classification accuracy.

In addition, we also consider networks trained on real data: A first one is designed and trained by ourselves on a {\em Sentiment Analysis} task where the graph represents a syntactic tree with nodes carrying information about each word. Another one is the {\em SchNet} \cite{schnet_paper} model used for molecular prediction where nodes and edges represent atom types and distances respectively. Finally, we use a  pre-trained {\em VGG-16} \cite{DBLP:journals/corr/SimonyanZ14a} network for image recognition that we interpret as a graph neural network operating on a lattice of size $14 \times 14$ and starting at convolutional block 3. In this last network, each node represents the collection of activation at a specific spatial location. Details for each network are given in Appendix F in the Supplement. Results for each network, explanation method and task are summarized in Tables \ref{table:activation} and \ref{table:pruning}.

\begin{table}[h]
\caption{AUFC scores on the {\em activation} task. The higher the score the better the explanation. Best performers are shown in bold. The results for SchNet-$E$ are scaled by a factor of $10^{-3}$.}
\label{table:activation}
\vspace{-3mm}
\begin{center}
\begin{tabular}{l|cc|cc|cc}\toprule
~ & 
\rotatebox{90}{P \cite{DBLP:conf/cvpr/PopeKRMH19} (GI)} &
\rotatebox{90}{P \cite{DBLP:conf/cvpr/PopeKRMH19} (LRP)} &
\rotatebox{90}{GNN-GI (ours)} &
\rotatebox{90}{GNN-LRP (ours)} &
\rotatebox{90}{GNNExpl\ \cite{DBLP:journals/corr/abs-1903-03894}} &
\rotatebox{90}{Random}\\\midrule
 Synth-GCN & 2.00 & 2.29 & 1.86 & \textbf{2.72} & 2.65 & 0.75 \\
 Synth-GIN & 2.96 & 3.60 & 3.23 & \textbf{3.92} & 3.81 & 1.39 \\
Synth-Spectral & \!-0.87\! & 1.39 & 0.87 & \textbf{2.29} & 1.30 & 0.04 \\
\midrule
Sentiment-GCN & \!19.34\! & \!19.58\! &  \!19.95\!  &  \!20.05\!  & \!{ \bf 20.23 }\! & \!15.27\! \\
\midrule
SchNet-$E$ \cite{schnet_paper} &  \multicolumn{2}{c|}{10.39} & \multicolumn{2}{c|}{\bf 10.47} & \!10.41\! & 8.00 \\
SchNet-$\mu$ \cite{schnet_paper} & \multicolumn{2}{c|}{0.87} & \multicolumn{2}{c|}{\bf 1.09} & 1.01 & 0.38 \\\midrule
VGG-16 \cite{DBLP:journals/corr/SimonyanZ14a} &  9.46 & \!13.18\! & \!12.03\! & \!\bf{14.04}\! & --- & 7.90 \\\bottomrule
\end{tabular}
\end{center}
\end{table}

\begin{table}[h]
\caption{AUFC scores on the {\em pruning} task. The lower the score the better the explanation. Best performers are shown in bold. The results for SchNet-$E$ are scaled by a factor of $10^{-3}$.}

\label{table:pruning}
\vspace{-3mm}
\begin{center}
\begin{tabular}{l|cc|cc|cc}\toprule
~ & 
\rotatebox{90}{P \cite{DBLP:conf/cvpr/PopeKRMH19} (GI)} &
\rotatebox{90}{P \cite{DBLP:conf/cvpr/PopeKRMH19} (LRP)} &

\rotatebox{90}{GNN-GI (ours)} &
\rotatebox{90}{GNN-LRP (ours)} &
\rotatebox{90}{GNNExpl  \cite{DBLP:journals/corr/abs-1903-03894}} &
\rotatebox{90}{Random}\\\midrule
  Synth-GCN & 2.36 & 2.55 & 2.20 & \textbf{1.93} & 2.10 & 3.39 \\
 Synth-GIN & 3.96 & 3.20 & 3.22 & \textbf{2.85} & 3.05 & 4.89 \\
Synth-Spectral & 4.35 & 2.78 & 2.58 & \textbf{2.22} & 2.61 & 3.77 \\
\midrule
Sentiment-GCN & \!21.43\! & \!21.16\! &   \!20.72\! & \!{ \bf 20.12 }\! & \!20.37\! &  \!22.89\! \\\midrule
SchNet-$E$ \cite{schnet_paper} & \multicolumn{2}{c|}{5.44} &   \multicolumn{2}{c|}{\bf 5.39} & 5.43 & 7.97\\
SchNet-$\mu$ \cite{schnet_paper} &  \multicolumn{2}{c|}{0.40} &  \multicolumn{2}{c|}{\bf 0.26} & 0.28 & 0.61 \\\midrule
VGG-16 \cite{DBLP:journals/corr/SimonyanZ14a} & \!10.02\! & 6.67 & 7.78 & \bf{6.05} & --- & \!11.47\!
 \\\bottomrule
\end{tabular}
\end{center}
\end{table}

On the {\em Synthetic Data}, we observe that GNN-LRP systematically outperforms other methods, both on the activation and pruning tasks. Nearest competitors are Pope et al.\ \cite{DBLP:conf/cvpr/PopeKRMH19} (in combination with LRP), GNN-GI, and the GNNExplainer \cite{DBLP:journals/corr/abs-1903-03894}. This corroborates our qualitative analysis at the beginning of Section \ref{section:validating}. 

On the {\em Sentiment Analysis} task, GNN-LRP and GNNExplainer \cite{DBLP:journals/corr/abs-1903-03894} perform best for the activation and pruning task, followed by GNN-GI and both node attribution methods by Pope et al. \cite{DBLP:conf/cvpr/PopeKRMH19}. In natural language, the sentiment associated to a sentence relies on word combinations, e.g.\ negation. This can explain why methods which attribute interactions of words such as GNN-LRP and GNNExplainer perform better.

For the experiments on the {\em SchNet}, our method performs again above competitors. Note that in this experiment, use of GI and LRP are equivalent, because the lack of ReLU-Dense-ReLU structures forces us to use the LRP-$0$ rule, instead of the more robust LRP-$\gamma$, and application of LRP-$0$ is equivalent to GI (cf.\ \cite{DBLP:journals/corr/ShrikumarGSK16,DBLP:series/lncs/Montavon19}). The difference to other competitors (Pope et al.\ \cite{DBLP:conf/cvpr/PopeKRMH19} and GNNExplainer \cite{DBLP:journals/corr/abs-1903-03894}) is small on the prediction of the energy $E$ but larger for the dipole-moment $\mu$, possibly because of a more complex structure of the prediction task involving longer interactions.

Finally, on the {\em VGG-16} image recognition model, GNN-LRP performs best on both tasks. Here, the superiority of LRP vs. GI can be explained both by a better handling of neuron biases and by a higher robustness to gradient shattering, a key difficulty to account for when explaining very deep models \cite{DBLP:conf/icml/BalduzziFLLMM17,DBLP:journals/dsp/MontavonSM18}.

Overall, we find that GNN-LRP is systematically the best method in our benchmark. With the fine-grained yet robust explanations it provides, GNN-LRP is capable of precisely and contextually identifying elements of the graph that contribute the most or the least to the prediction.

\section{New Insights with GNN-LRP}

Having validated the proposed GNN-LRP method on a diverse set of GNNs including state-of-the-art models, and having shown the multiple advantages of our method compared to previous approaches, we will now inspect the explanations produced by GNN-LRP on some of these practically relevant GNN models to demonstrate how useful insights can be extracted about the GNN model and the task it predicts.

\subsection{Sentiment Analysis}
\label{section:sentiment}

In natural language processing (NLP) text data can be processed either as a sequence, or with its corresponding grammatical structure represented by a parse tree \cite{Jurafsky2009, rieck2010approximate}. The latter serves as an additional structural input for the learning algorithm, to incorporate dependencies between words. NLP tasks are therefore particularly amenable to GNNs since these models can naturally incorporate the graph structure.

In the following experiments, we will demonstrate how GNN-LRP can be used to intuitively and systematically assess the quality of a GNN model, including its overall prediction strength and also its few weaknesses. For this, we will consider a GCN composed of two interaction layers, to classify sentiments in natural language text \cite{DBLP:series/synthesis/2012Liu}. We train our model on the Stanford Sentiment Treebank (SST) \cite{socher-etal-2013-recursive}\footnote{Note that this example serves as a mere demonstration for the versatility of our explanation approach and is by no means intended to reflect or compete with state-of-the-art NLP systems.}. For details on the experimental setup we refer the reader to Appendix F.2 in the Supplement.

In Fig.\ \ref{fig:sst_relevances} (top) we show an example of a GNN-LRP explanation for some exemplary input sentence containing a mixture of positive and negative sentiment. We observe that distinct combinations of words give rise to the emphasis of these sentiments. The model correctly detects word combinations such as ``\textit{the best movies}'' to carry a positive sentiment, and ``\textit{boring pictures}'' to contribute negatively.

\begin{figure}[h]
	\centering
    \normalsize \sffamily
     \parbox{1\linewidth} {\centering GNN}\\
    \includegraphics[width=\linewidth]{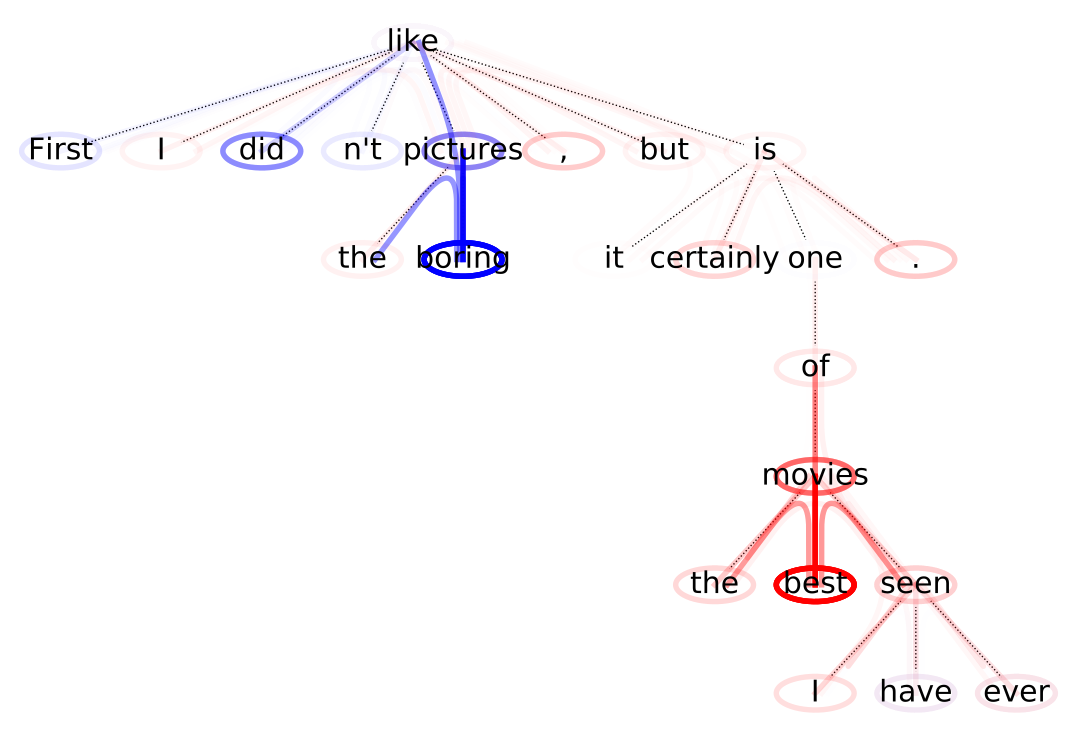}\\
     \vskip -3.3cm \hskip -3.5cm
     \parbox{.6\linewidth}{\centering  BoW } 
     \vskip 0.2cm \hskip -3.5cm
     \includegraphics[width=.6\linewidth]{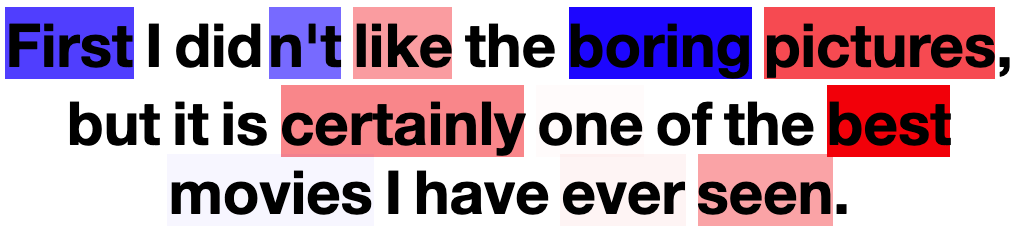}
     \vskip 1cm
	\caption{Sentence predicted by the GNN and the BoW model, and explained by GNN-LRP (applied on the difference between the positive and negative sentiment logit, and with the LRP parameter $\gamma = 3$). Contributions to positive sentiment are in red, and contributions to negative sentiment are in blue. 
	}\label{fig:sst_relevances}
\end{figure}

GNN-LRP can also be used to assess the GNN prediction strategy relative to other (simpler) models, such as Bag-of-Words (BoW). The BoW model can be seen as a GNN model with zero interaction layers, hence our explanation technique applies to that model as well. Fig.\ \ref{fig:sst_relevances} (bottom) shows the explanation of the BoW prediction for the same sentence as for the GNN. Interestingly, several words are now attributed a sentiment different from the one obtained with the GNN. For example, ``\textit{like}'' becomes positive, which however appears in contradiction with the preceding words ``\textit{didn't}''. Hence, GNN-LRP has highlighted from a single sentence that the GNN model is able to properly capture and disambiguate the sentiment of consecutive words, whereas the BoW model is not.

\begin{figure}[h]
    \centering
    \footnotesize \sffamily
     \includegraphics[width=.85\linewidth]{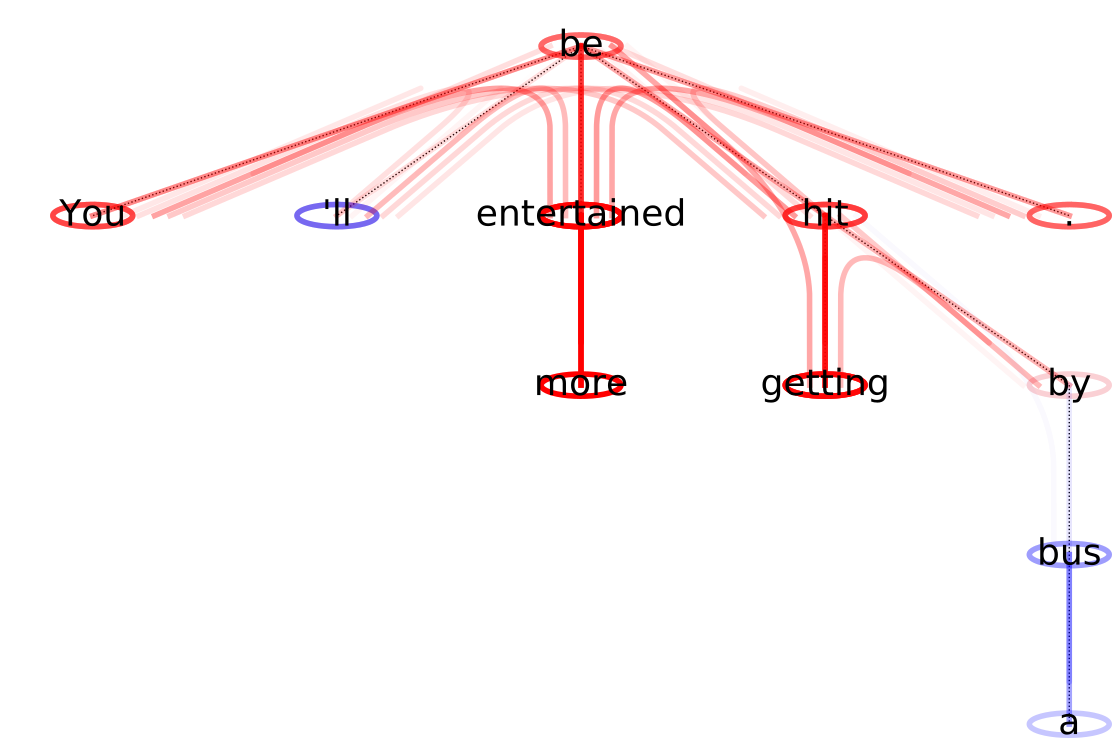}\\
    \vskip -3.2cm
     \hfill \parbox{.4\linewidth} {\flushleft A. Incorrect prediction:\\[1mm]\textit{``You'll be more entertained getting hit by a bus.''\\$\to$
    \textbf{positive}}}
    \parbox{.575\linewidth}{~}\\
    \hskip -1cm
    \includegraphics[width=.85\linewidth,clip=True]{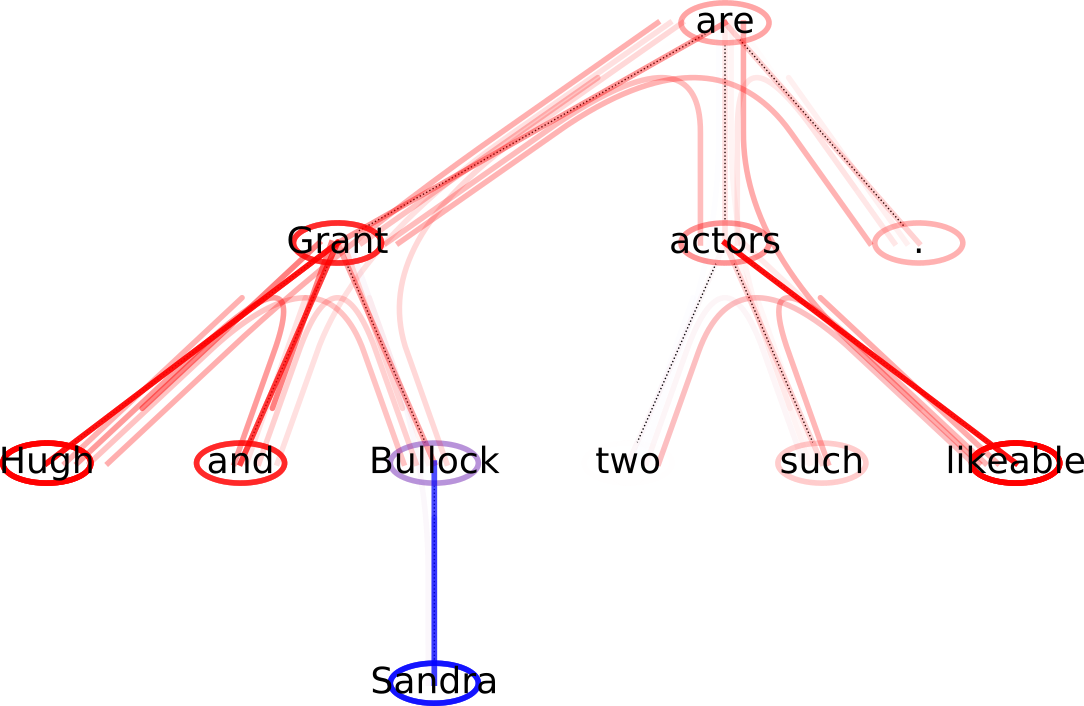}
    \vskip -1.2cm \hskip 4.5cm
    \parbox{.4\linewidth} {\flushleft B. Entity bias:\\[1mm]\textit{``Hugh Grant and Sandra Bullock are two such likeable actors.''$\to$
    \textbf{positive}}} \\
    \caption{Two selected examples of dependency trees from the SST dataset, predicted by the GNN model to be positive, but for which GNN-LRP highlights a flaw in the prediction strategy.}
    \label{fig:add_nlp_samples}
\end{figure}

In the next experiment, we consider two additional sentences from the SST corpus, where GNN-LRP contributes to uncovering or better understanding flaws of a trained GNN model. Fig.\ \ref{fig:add_nlp_samples}\,A shows a data sample that contains sarcasm and that is falsely classified by the GNN to be positive. Our explanation method highlights that relevant walks are too localized to capture the interaction between words that jointly explain sarcasm. Instead, local positive interactions such as ``\textit{more entertained}'' dominate the prediction, which leads to the incorrect prediction. In Fig.\ \ref{fig:add_nlp_samples}\,B, we see a case of {\em entity bias} where the GNN model is biased towards particular entities, namely ``\textit{Hugh Grant}'' and ``\textit{Sandra Bullock}''. In this example, GNN-LRP finds that the GNN model uses with no objective reason ``\textit{Hugh Grant}'' and ``\textit{Sandra Bullock}'' as evidence for positive and negative sentiment, respectively. In NLP model biases are well studied areas and some approaches to tackle that problem have already been developed \cite{NIPS2016_6228,DBLP:journals/corr/abs-1903-03862}.

To identify words or combinations of words that systematically contribute to a positive or negative sentiment---and potentially discover further cases of entity bias,---we apply GNN-LRP on the whole dataset. This lets us find the combination of words (given by a walk $\mathcal{W}$) that are on average the most positive/negative (according to their score $R_\mathcal{W}$). Fig.\ \ref{fig:topsentiment} shows top-3 walks of both types (positive and negative) and containing one to three unique words.

\begin{figure}[h]
    \centering \sffamily \footnotesize
    \scalebox{.85}[1.1]{%
    \begin{minipage}{.16\linewidth} \em
    solidly\\
    wonderfully \\
    lewis\\
    \hspace*{1mm}\vdots\\
    vulgar\\
    gimmicky\\
    moot
    \end{minipage}}
    \hfill \vline \hfill
    \scalebox{.85}[1.1]{%
    \begin{minipage}{.34\linewidth} \em
    documentary brilliant\\
    provocative entertaining \\
    provocative boldly\\
    \hspace*{1mm}\vdots\\
    dead weight\\
    no more\\
    charlie sorry
    \end{minipage}}
    \hfill \vline \hfill
    \scalebox{.85}[1.1]{%
    \begin{minipage}{.51\linewidth} \em
    documentary disturbing brilliant\\
    entertainment resourceful ingenious \\
    dyspeptic touching wonderfully\\
    \hspace*{1mm}\vdots\\
    instead of becomes\\
    bears bad is   \\
    no more .
    \end{minipage}}
    \caption{Walks that contribute the most to positive and negative sentiment according to the GNN, split by the number of unique words they contain.}
    \label{fig:topsentiment}
\end{figure}

We see that the walks which are most relevant for the task contain positive adjectives and adverbs, such as ``\textit{solidly}'', ``\textit{brilliant}'', ``\textit{wonderfully}'' or ``\textit{entertaining}''. The walks with a very negative relevance score contain negative words such as ``\textit{moot}'', ``\textit{gimmicky}'' or ``\textit{vulgar}'', but also subsequences like ``\textit{charlie sorry}'', ``\textit{dead weight}'' or ``\textit{no more.}'' which clearly transport a negative emotion. Here again, GNN-LRP detects an entity bias by the word ``\textit{lewis}'', which the GNN model considers to be positively contributing although this word is objectively neutral. Note that this time, this entity bias was discovered directly, without having to visualize a large number of explanations.

Overall, applying the proposed GNN-LRP explanation method to the GNN model for sentiment classification has highlighted that GNN predictions are based on detecting meaningful sentence sub-structures, rather than single words as in the BoW model. Furthermore, GNN-LRP was able to find the reasons for incorrect predictions, or to shed light on potential model biases. The latter could be identified manually by visual inspection of many explanations, or systematically by averaging the GNN-LRP results on a whole corpus.

\subsection{Quantum Chemistry}
\label{section:qc}

In the field of machine learning for quantum chemistry \cite{von2019exploring, noe2020machine, QML-Book}, GNNs have been exhibiting state of the art performance for predicting molecular properties \cite{schutt2017quantum, DBLP:conf/icml/GilmerSRVD17, schnet_paper, schutt2018schnet}. Such networks incorporate a graph structure of molecules either by the covalent bonds or the proximity of atoms.

In this section we will test the ability of GNN-LRP to extract meaningful domain knowledge from these state-of-the-art GNNs. We will consider for this the {\em SchNet}, a GNN for the prediction of molecular properties \cite{schnet_paper, schutt2018schnet, schutt2018schnetpack}, and we set the number of interaction blocks in this GNN to three. We train the model on the atomization energies and the dipole moments for 110,000 randomly selected molecules in the QM9 dataset~\cite{ramakrishnan2014quantum}. For more details on the model parameters and the network architecture, we refer to Appendix F.3 in the Supplement. On a test set of 13,885 molecules, the atomization energy and dipole moment are predicted well with a mean absolute error (MAE) of 0.015\,eV and 0.039\,Debye, respectively. 

Our first objective is to get an insight into what structures in the molecule contribute positively or negatively to the molecule's energy. While it is common to look at the atomization energy (describing the energy difference to dissociated atoms), we consider here for the purpose of explanation the centered negative atomization energy, and we define this quantity to be the actual `energy'. With this definition, molecules have high energy when they are hard to break and typically formed of strong bonds, and conversely, molecules have low energy when they are easy to break and unstable.

We consider for illustration the case of the \textit{paracetamol} (acetaminophen) molecule, and feed this molecule to SchNet. Once the SchNet model has predicted its energy, we apply the GNN-LRP analysis in order to produce an explanation of the prediction. The resulting explanation is shown in  Fig.~\ref{fig:qc_walks_energy}~(left). We observe that the explanation is dominated by self-walks (i.e.\ staying in a single atom) or one-edge walks (traversing a single edge of the graph, in most cases, a bond). Bonds associated to the aromatic ring and bonds of higher order contribute strongly to the predicted energy, whereas regions involving single bonds contribute negatively. To verify whether this observation generalizes to other molecules, we perform the GNN-LRP analysis on a set of 1000 molecules randomly drawn from the QM9 dataset and show in Fig.~\ref{fig:qc_walks_energy}~(right) the average bond contribution for each bond type. We observe an increasing energy contribution with ascending bond order. This coincides with chemical intuition that bonds of higher order are more stable and, thus, require more energy to break.---Note that the SchNet does not take bond types as an input, but as highlighted by GNN-LRP, it has clearly inferred these chemical features from the data.

\begin{figure}[t!]
	\centering
	\includegraphics[width=0.95\linewidth]{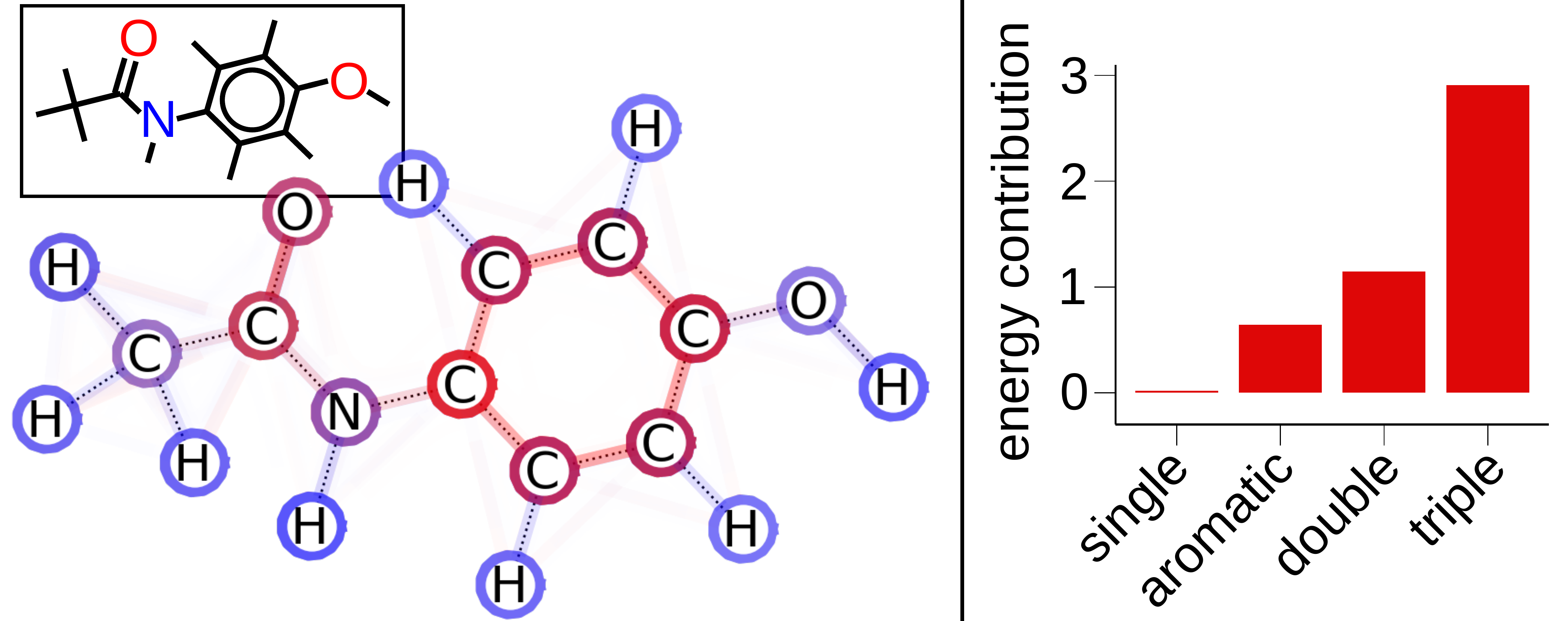}
	\caption{{\em Left:}~Paracetamol molecule, and the GNN-LRP explanation of its predicted energy. Red and blue indicate to positive and negative contributions. Opacity indicates the magnitude of these contributions. {\em Right:}~Average energy contribution per bond, depicted for each bond type separately, in arbitrary units.}
	\label{fig:qc_walks_energy}
\end{figure}

We now turn to another quantum chemical property, the \textit{dipole moment}, and its prediction by  SchNet  (cf.\ \cite{Schutt2019book}). The quantity produced at the output of the model has the form $\|\bm{\mu}(\bm{\Lambda})\|$. The nonlinearity of the norm introduces higher-order terms in the GNN function and this prevents a direct application of GNN-LRP. Instead, we consider for explanation the dot product $\langle \bm{\mu}(\bm{\Lambda}), \bm{\mu}[\bm{\Lambda}^\star]/\|\bm{\mu}[\bm{\Lambda}^\star]\|\rangle$, where the left hand side functionally depends on the GNN input $\bm{\Lambda}$, and where the right hand side is the direction of the predicted dipole moment for the actual molecule denoted by $\bm{\Lambda}^\star$. With this modification, the top layer becomes linear, and GNN-LRP can proceed as for the energy prediction case. 

Fig.~\ref{fig:dipole_statistic} (top) shows the GNN-LRP explanation of the predicted dipole moment for the same molecule as in the previous experiment. Here, we observe that the contributions to the dipole moment found by GNN-LRP form a gradient from one side to the other side of the molecule. The orientation corresponds to the positive and negative pole of the molecule. To verify whether this insight generalizes to other molecules, we consider a set of 1000 molecules and we normalize each molecule to a span of 1 along its dipole direction. Subsequently, we project all walks onto their respective dipole to obtain a one-dimensional distribution of absolute relevance values for all molecules. Note that the atom density of molecules, in general, is not homogeneous, and thus, in some regions more walks may occur while other regions do not exhibit as many walks. Hence, for each molecule the distribution of absolute relevance is normalized w.r.t. its atom density. The result of this aggregated analysis is shown in  Fig.~\ref{fig:dipole_statistic} (bottom).

\begin{figure}[h]
	\centering
	\includegraphics[ width=0.95\linewidth]{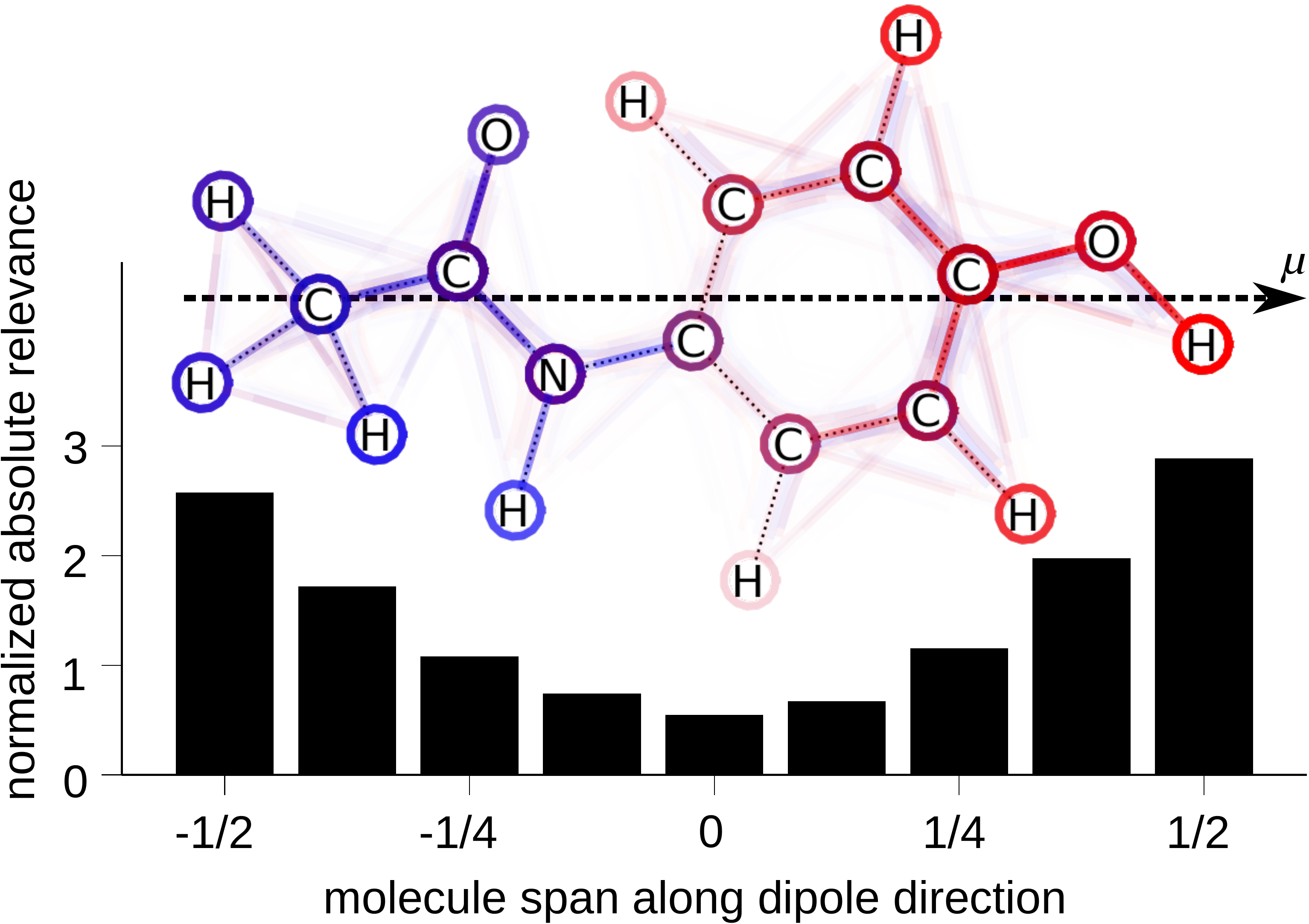}
	\caption{{\em Top:} Paracetamol molecule with the predicted dipole direction shown as a dotted arrow, and the GNN-LRP explanation. {\em Bottom:}~Distribution of contributions (in absolute terms) along the direction of the predicted dipole moment, averaged over the dataset.}\label{fig:dipole_statistic}
\end{figure}

This quantitative result confirms the alignment of GNN-LRP contributions with the positive and negative poles of the molecule, as it was found qualitatively on the paracetamol molecule. This result is in accordance with chemical intuition regarding the dependence of the dipole on the span of the molecule.

Because the walk-based explanations produced by GNN-LRP are very detailed, some of the more intricate details of the explanation cannot always be visualized on a single molecule. Hence, we perform a further experiment where the GNN-LRP explanation is spread over multiple visuals, each of them showing relevant walks covering a specific number of edges. With this expanded visualization, we seek in particular to better distinguish between local atom-wise contribution and more global effects. 
Figure \ref{fig:interaction_orders} shows this expanded analysis of the dipole moment prediction for the paracetamol molecule.

\begin{figure}[h]
	\centering
	\includegraphics[width=\linewidth]{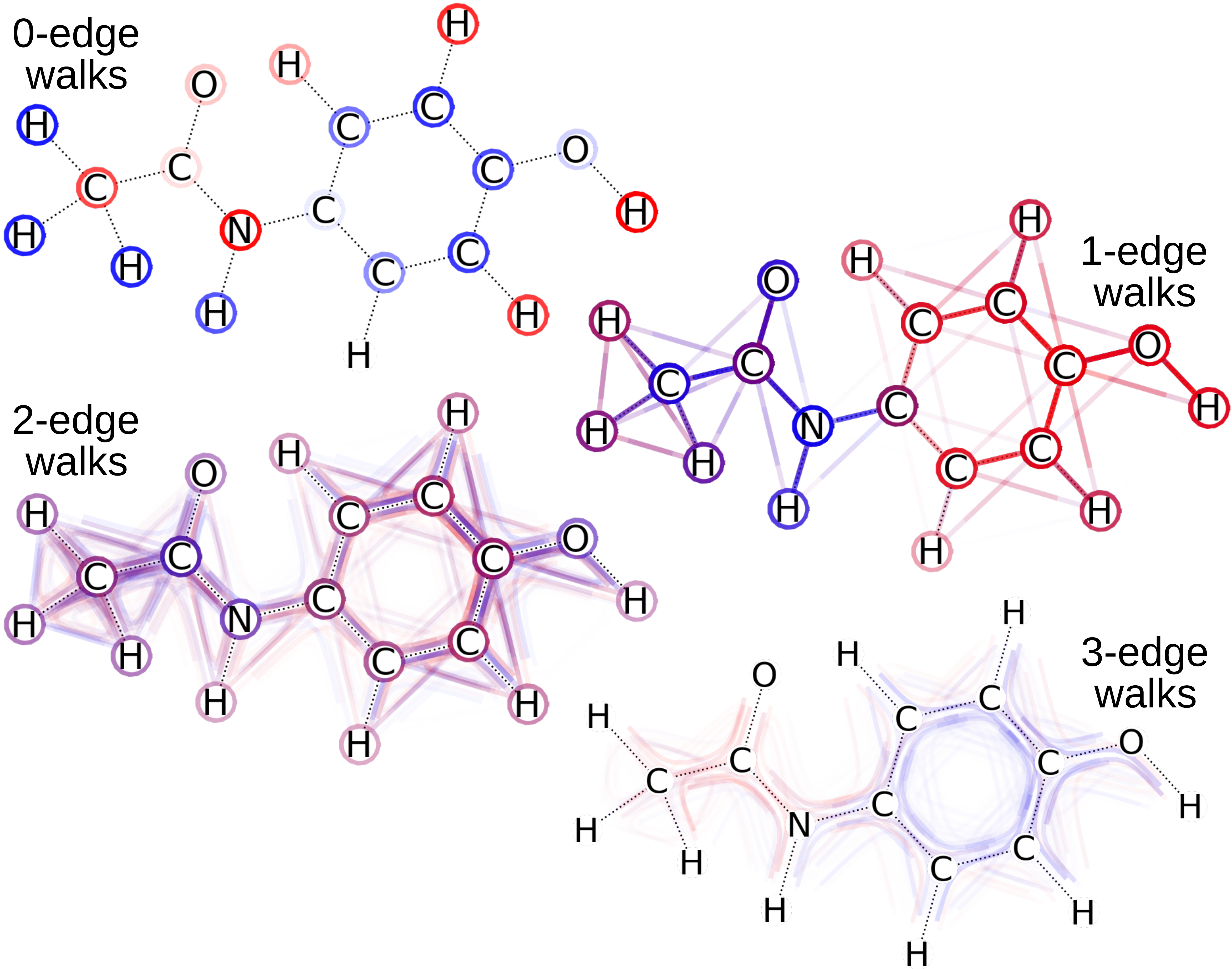}
	\caption{Expanded GNN-LRP explanation for the dipole moment prediction of the paracetamol molecule. To increase visibility of two-edge walks and three-edge walks, we show the scaled relevance scores $R' = R^{0.7}$}
	\label{fig:interaction_orders}
\end{figure}

From this visualization, we gain further insights into the strategy used by the SchNet model for predicting. In our expanded explanation, one-edge walks clearly indicate the electrostatic poles of the molecule, while giving a hint on local dipoles. Self-walks (i.e.\ 0-edge walks) incorporate elements that are inherent to the atom types, in particular, their electronegativity. Two-edge walks provide rather complex and spatially less resolved contributions. Finally, three-edge walks, that are also the most global descriptors, again provide interesting spatial contributions, and appear to dampen the one-edge walks contributions based on the more complex structures they are able to capture.

Overall, GNN-LRP has provided insights into the structure-property relationship of molecules that reach beyond the original prediction task. The resulting relevant walks agree with chemical characteristics of the molecule, thus, indicating that the neural network has indeed learned chemically plausible regularities.

\subsection{Revisiting Image Classification}
\label{section:vgg}

\begin{figure*}[h!]
\includegraphics[width=\textwidth]{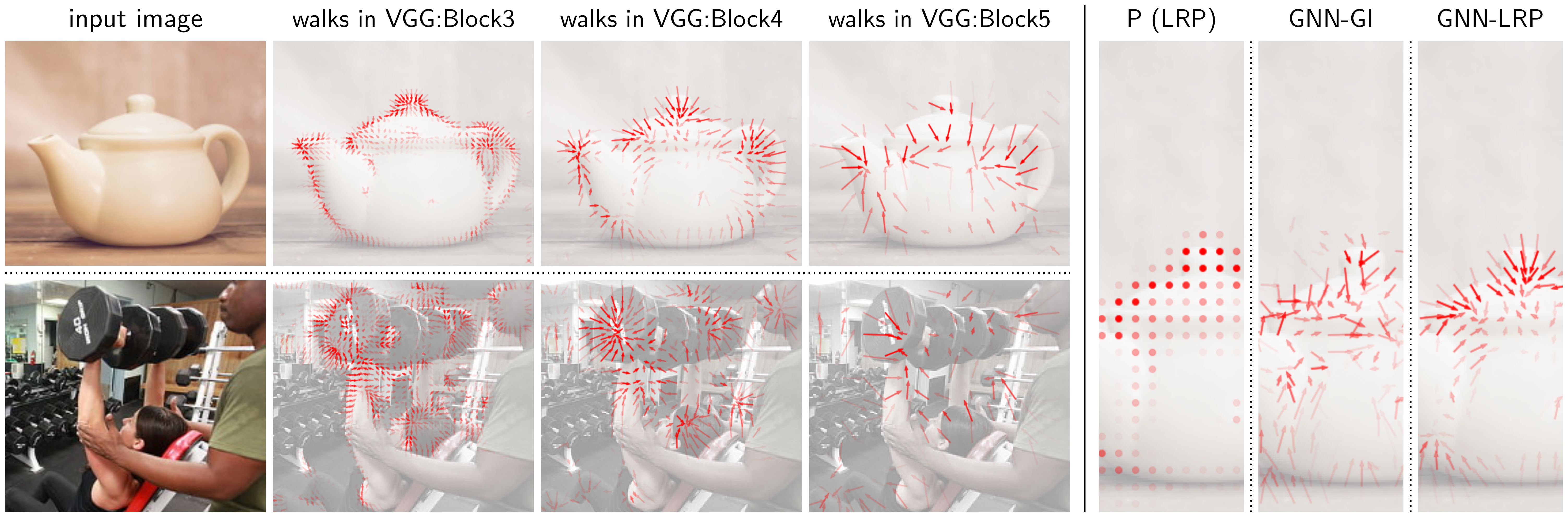}
\caption{{\em Left:} Relevant walks in the pixel lattice explaining the prediction by the VGG-16 network of two input images as `teapot' and `dumbbell' respectively. In each vector field, arrows connect block input nodes to the relevance-weighted average position of the block output nodes. {\em Right:} Comparison of GNN-LRP with different explanation techniques on Block 4.}
\label{fig:vgg_walks}
\end{figure*}

A convolutional neural network (CNN) can be seen as a particular graph neural network (GNN) operating on lattices of pixels. CNN predictions have so far mainly been explained using heatmaps highlighting pixels that are the most relevant for a given prediction \cite{DBLP:conf/eccv/ZeilerF14,bach-plos15,IntegratedGradient_SundararajanTY17}. Heatmaps are a useful representation summary of the decision structure, but they do not reveal the more complex strategies of a network that have been used to progressively build the prediction layer after layer. We will show by viewing CNNs as graph neural networks and extracting relevant walks in the resulting pixel lattice, that our GNN-LRP method is capable of shedding light into these strategies.

For this, we consider the well-established VGG-16 \cite{DBLP:journals/corr/SimonyanZ14a} network. It consists of a collection of blocks interleaved by pooling layers, where each block is composed of a sequence of convolution and ReLU layers. We use the pretrained version of the VGG-16 network \cite{DBLP:journals/corr/SimonyanZ14a} without batch-normalization, which can be retrieved using the TorchVision module of PyTorch.

Because the VGG-16 neural network is deep and the number of possible walks grows exponentially with neural network depth, we marginalize explanations to only consider the position of the walk at the input and at the output of a block. This is easily achieved by removing all masks except those at the input and output of the block (cf.\ Section \ref{section:implementation}). We then compute one explanation for block 3, 4, and 5. Also, to cope with the large spatial lattices in each block, we make use of the multi-mask strategy outlined in Appendix D of the Supplement. Specifically, because each block of the VGG-16 network has receptive fields of size $7$, we can process multiple walks at the same time by choosing the mask to be a grid with stride $7$. This allows us to collect all relevant walks at the given block in the order of $49$ backward passes.

We consider two exemplary images\footnote{Images are from \url{https://www.piqsels.com/en/public-domain-photo-fjjsr} and \url{https://www.piqsels.com/en/public-domain-photo-fiffy}, rescaled and cropped to the relevant region to produce images of size $224\times 224$ which are the standard input size for VGG-16.} that the VGG-16 network respectively predicts as `teapot' and `dumbbell'. We set the LRP parameter to $\gamma=0.5$ in block $3$, halving the parameter value in each subsequent block, and choosing $\gamma=0$ in the top-level classifier. Fig.\ \ref{fig:vgg_walks} (left) shows the result of the analysis for these two images at various blocks of the VGG-16 network.

For the first image, Block 3 detects local edges in the teapot, then, in Block 4, the walks converge to center points of specific parts of the teapot (e.g.\ the handle, the spout and the knob), and finally the walks converge in Block 5 to the center of the teapot, which can be interpreted as composing the different parts of the teapot. For this exemplary image, we further observe in Fig.\ \ref{fig:vgg_walks} (right) the advantageous properties of GNN-LRP compared to more basic explanation methods. The GNN-GI baseline also produces a vector field, however, the latter is significantly more noisy than the one produced by GNN-LRP. The method by Pope et al.\ \cite{DBLP:conf/cvpr/PopeKRMH19} robustly highlights relevant nodes at the input of the given block, however, it does not reveal where these features are being transported for use in the subsequent block.

For the second image, we investigate a known `Clever Hans' strategy where the network classifies images as `dumbbell' by detecting both the dumbbell and the arm that holds it \cite{inceptionism}. Using GNN-LRP we observe that Blocks 3 and 4 detect the arm and the dumbbell separately, and then, Block 5 composes them into a single `dumbbell-arm' concept, as shown by the walks for both objects converging to some center point near the wrist. Clearly these insights could not have been obtained from a standard pixel-wise heatmap explanation.

Overall, our GNN-LRP method can be used to comprehensively inspect the prediction of an image classifier beyond what would be possible with a standard pixel-wise heatmap explanation. This deeper explanation capability allows us to better understand the detailed structure of image classifications, and also to shed more light into anecdotal `Clever Hans' effects observed in the context of state-of-the-art image classifiers.

\section{Conclusion}

Graph neural networks are a highly promising approach for predicting graphs, with a strong demand from the practical side. For these models to be broadly adopted, it is however important that their predictions are made explainable to the user.---Because the input of a GNN is tightly entangled with the model itself, the explanation problem is particularly difficult, and, methods for explaining GNNs have so far been limited.

In this paper we have proposed a novel theoretically principled approach to produce these explanations, based higher-order Taylor expansions. From this conceptual starting point, we have then contributed two practical algorithms: GNN-GI which we propose as a simple baseline, and GNN-LRP which is more robust and scales to highly non-linear models. GNN-LRP produces detailed explanations that subsume the complex nested interaction between the GNN model and the input graph. It also significantly outperforms other explanation methods in our quantitative benchmark.---In addition to the high quality of the explanations it produces, GNN-LRP is also broadly applicable (covering the GCN, the GIN, and spectral filtering approaches), and it can handle virtually any type of input graph, whether it is a parse tree, a spatial graph, a pixel lattice, etc. 

This broad applicability is demonstrated in our extensive application showcase, including sentiment analysis, quantum chemistry and image classification. In each scenario, GNN-LRP could highlight the diverse strategies employed by the GNN models, and unmask some undesired `Clever Hans' strategies. In our quantum-chemical application showcase, we could additionally extract interesting problem-relevant insights. Future work will apply the proposed methodology to analyze properties of  materials for practically highly relevant tasks, e.g.\ in catalysis.

\section*{Acknowledgments}

This work was funded by the German Ministry for Education and Research as BIFOLD -- Berlin Institute for the Foundations of Learning and Data (ref.\ 01IS18025A and ref.\ 01IS18037A), and the German Research Foundation (DFG) as Math+: Berlin Mathematics Research Center  (EXC 2046/1, project-ID: 390685689). This work was partly supported by the 
Institute of Information \& Communications Technology Planning \& Evaluation (IITP) grants funded by the Korea Government (No. 2019-0-00079,  Artificial Intelligence Graduate School Program, Korea University).  Furthermore we would like to thank Jacob Kauffmann, Lukas Ruff and Marina H\"ohne for the highly helpful and illuminating discussions on the topic and also Niklas W. A. Gebauer for helping with the molecule visualization. We would like to thank Jasmijn Bastings for very helpful comments on the manuscript.

\bibliographystyle{IEEEtran}
\bibliography{gnn}

\begin{thebibliography}{10}
\providecommand{\url}[1]{#1}
\csname url@samestyle\endcsname
\providecommand{\newblock}{\relax}
\providecommand{\bibinfo}[2]{#2}
\providecommand{\BIBentrySTDinterwordspacing}{\spaceskip=0pt\relax}
\providecommand{\BIBentryALTinterwordstretchfactor}{4}
\providecommand{\BIBentryALTinterwordspacing}{\spaceskip=\fontdimen2\font plus
\BIBentryALTinterwordstretchfactor\fontdimen3\font minus
  \fontdimen4\font\relax}
\providecommand{\BIBforeignlanguage}[2]{{%
\expandafter\ifx\csname l@#1\endcsname\relax
\typeout{** WARNING: IEEEtran.bst: No hyphenation pattern has been}%
\typeout{** loaded for the language `#1'. Using the pattern for}%
\typeout{** the default language instead.}%
\else
\language=\csname l@#1\endcsname
\fi
#2}}
\providecommand{\BIBdecl}{\relax}
\BIBdecl

\bibitem{Scarselli:2009:GNN:1657477.1657482}
F.~Scarselli, M.~Gori, A.~C. Tsoi, M.~Hagenbuchner, and G.~Monfardini, ``The
  graph neural network model,'' \emph{{IEEE} Trans. Neural Networks}, vol.~20,
  no.~1, pp. 61--80, 2009.

\bibitem{Wu2020}
Z.~Wu, S.~Pan, F.~Chen, G.~Long, C.~Zhang, and P.~S. Yu, ``A comprehensive
  survey on graph neural networks,'' \emph{{IEEE} Transactions on Neural
  Networks and Learning Systems}, pp. 1--21, 2020.

\bibitem{schutt2018schnet}
K.~T. Sch{\"u}tt, H.~E. Sauceda, P.-J. Kindermans, A.~Tkatchenko, and K.-R.
  M{\"u}ller, ``Sch{N}et--a deep learning architecture for molecules and
  materials,'' \emph{The Journal of Chemical Physics}, vol. 148, no.~24, p.
  241722, 2018.

\bibitem{DBLP:journals/bioinformatics/ZitnikAL18}
M.~Zitnik, M.~Agrawal, and J.~Leskovec, ``Modeling polypharmacy side effects
  with graph convolutional networks,'' \emph{Bioinform.}, vol.~34, no.~13, pp.
  i457--i466, 2018.

\bibitem{DBLP:conf/emnlp/MarcheggianiT17}
D.~Marcheggiani and I.~Titov, ``Encoding sentences with graph convolutional
  networks for semantic role labeling,'' in \emph{{Empirical Methods in Natural
  Language Processing}}.\hskip 1em plus 0.5em minus 0.4em\relax Association for
  Computational Linguistics, 2017, pp. 1506--1515.

\bibitem{DBLP:conf/emnlp/BastingsTAMS17}
J.~Bastings, I.~Titov, W.~Aziz, D.~Marcheggiani, and K.~Sima'an, ``Graph
  convolutional encoders for syntax-aware neural machine translation,'' in
  \emph{{Empirical Methods in Natural Language Processing}}.\hskip 1em plus
  0.5em minus 0.4em\relax Association for Computational Linguistics, 2017, pp.
  1957--1967.

\bibitem{DBLP:conf/acl/CohnHB18}
D.~Beck, G.~Haffari, and T.~Cohn, ``Graph-to-sequence learning using gated
  graph neural networks,'' in \emph{Proceedings of the 56th Annual Meeting of
  the Association for Computational Linguistics, {ACL}}.\hskip 1em plus 0.5em
  minus 0.4em\relax Association for Computational Linguistics, 2018, pp.
  273--283.

\bibitem{DBLP:journals/tog/WangSLSBS19}
Y.~Wang, Y.~Sun, Z.~Liu, S.~E. Sarma, M.~M. Bronstein, and J.~M. Solomon,
  ``Dynamic graph {CNN} for learning on point clouds,'' \emph{{Association for
  Computing Machinery} Trans. Graph.}, vol.~38, no.~5, pp. 146:1--146:12, 2019.

\bibitem{lapuschkin2019unmasking}
S.~Lapuschkin, S.~W{\"a}ldchen, A.~Binder, G.~Montavon, W.~Samek, and K.-R.
  M{\"u}ller, ``Unmasking {C}lever {H}ans predictors and assessing what
  machines really learn,'' \emph{Nature communications}, vol.~10, p. 1096,
  2019.

\bibitem{DBLP:series/lncs/11700}
W.~Samek, G.~Montavon, A.~Vedaldi, L.~K. Hansen, and K.-R. M{\"{u}}ller, Eds.,
  \emph{Explainable {AI:} Interpreting, Explaining and Visualizing Deep
  Learning}, ser. Lecture Notes in Computer Science.\hskip 1em plus 0.5em minus
  0.4em\relax Springer, 2019, vol. 11700.

\bibitem{Ribeiro:2016:LIME}
M.~T. Ribeiro, S.~Singh, and C.~Guestrin, ``"{W}hy should {I} trust you?":
  Explaining the predictions of any classifier,'' in \emph{Proceedings of the
  22nd {ACM} {SIGKDD} International Conference on Knowledge Discovery and Data
  Mining}.\hskip 1em plus 0.5em minus 0.4em\relax {Association for Computing
  Machinery}, 2016, pp. 1135--1144.

\bibitem{IntegratedGradient_SundararajanTY17}
M.~Sundararajan, A.~Taly, and Q.~Yan, ``Axiomatic attribution for deep
  networks,'' in \emph{Proceedings of the 34th International Conference on
  Machine Learning}, vol.~70, 2017, pp. 3319--3328.

\bibitem{bach-plos15}
S.~Bach, A.~Binder, G.~Montavon, F.~Klauschen, K.-R. M{\"u}ller, and W.~Samek,
  ``On pixel-wise explanations for non-linear classifier decisions by
  layer-wise relevance propagation,'' \emph{PLoS ONE}, vol.~10, no.~7, p.
  e0130140, 2015.

\bibitem{DBLP:journals/corr/Samek-XAI-review}
W.~Samek, G.~Montavon, S.~Lapuschkin, C.~J. Anders, and K.-R. M{\"{u}}ller,
  ``Interpretable machine learning: Transparent deep neural networks and
  beyond,'' \emph{CoRR}, vol. abs/2003.07631, 2020.

\bibitem{DBLP:journals/inffus/ArrietaRSBTBGGM20}
A.~B. Arrieta, N.~D. Rodr{\'{\i}}guez, J.~D. Ser, A.~Bennetot, S.~Tabik,
  A.~Barbado, S.~Garc{\'{\i}}a, S.~Gil{-}Lopez, D.~Molina, R.~Benjamins,
  R.~Chatila, and F.~Herrera, ``Explainable artificial intelligence {(XAI):}
  concepts, taxonomies, opportunities and challenges toward responsible {AI},''
  \emph{Inf. Fusion}, vol.~58, pp. 82--115, 2020.

\bibitem{DBLP:journals/corr/abs-1812-08434}
J.~Zhou, G.~Cui, Z.~Zhang, C.~Yang, Z.~Liu, and M.~Sun, ``Graph neural
  networks: {A} review of methods and applications,'' \emph{CoRR}, vol.
  abs/1812.08434, 2018.

\bibitem{Eberle2020}
O.~Eberle, J.~Buttner, F.~Krautli, K.-R. M{\"u}ller, M.~Valleriani, and
  G.~Montavon, ``Building and interpreting deep similarity models,''
  \emph{{IEEE} Trans. Pattern Anal. Mach. Intell.}, 2020.

\bibitem{DBLP:journals/corr/Janizek2020}
J.~D. Janizek, P.~Sturmfels, and S.~Lee, ``Explaining explanations: Axiomatic
  feature interactions for deep networks,'' \emph{CoRR}, vol. abs/2002.04138,
  2020.

\bibitem{DBLP:conf/ecai/CuiMK20}
T.~Cui, P.~Marttinen, and S.~Kaski, ``Learning global pairwise interactions
  with bayesian neural networks,'' in \emph{24th European Conference on
  Artificial Intelligence}, ser. Frontiers in Artificial Intelligence and
  Applications, vol. 325.\hskip 1em plus 0.5em minus 0.4em\relax {IOS} Press,
  2020, pp. 1087--1094.

\bibitem{DBLP:conf/kdd/CaruanaLGKSE15}
R.~Caruana, Y.~Lou, J.~Gehrke, P.~Koch, M.~Sturm, and N.~Elhadad,
  ``Intelligible models for healthcare: Predicting pneumonia risk and hospital
  30-day readmission,'' in \emph{Proceedings of the 21th {ACM} {SIGKDD}
  International Conference on Knowledge Discovery and Data Mining}.\hskip 1em
  plus 0.5em minus 0.4em\relax {Association for Computing Machinery}, 2015, pp.
  1721--1730.

\bibitem{DBLP:conf/iclr/TsangC018}
M.~Tsang, D.~Cheng, and Y.~Liu, ``Detecting statistical interactions from
  neural network weights,'' in \emph{6th International Conference on Learning
  Representations}.\hskip 1em plus 0.5em minus 0.4em\relax OpenReview.net,
  2018.

\bibitem{zhang2018graph}
Y.~Zhang, P.~Qi, and C.~D. Manning, ``Graph convolution over pruned dependency
  trees improves relation extraction,'' in \emph{{Empirical Methods in Natural
  Language Processing}}.\hskip 1em plus 0.5em minus 0.4em\relax Association for
  Computational Linguistics, 2018, pp. 2205--2215.

\bibitem{DBLP:conf/cvpr/PopeKRMH19}
P.~E. Pope, S.~Kolouri, M.~Rostami, C.~E. Martin, and H.~Hoffmann,
  ``Explainability methods for graph convolutional neural networks,'' in
  \emph{{IEEE} Conference on Computer Vision and Pattern Recognition}.\hskip
  1em plus 0.5em minus 0.4em\relax Computer Vision Foundation / {IEEE}, 2019,
  pp. 10\,772--10\,781.

\bibitem{pub10600}
R.~Schwarzenberg, M.~H{\"{u}}bner, D.~Harbecke, C.~Alt, and L.~Hennig,
  ``Layerwise relevance visualization in convolutional text graph
  classifiers,'' in \emph{Proceedings of the Thirteenth Workshop on Graph-Based
  Methods for Natural Language Processing, TextGraphs@EMNLP}.\hskip 1em plus
  0.5em minus 0.4em\relax Association for Computational Linguistics, 2019, pp.
  58--62.

\bibitem{DBLP:journals/corr/abs-1903-03894}
Z.~Ying, D.~Bourgeois, J.~You, M.~Zitnik, and J.~Leskovec, ``{GNNE}xplainer:
  Generating explanations for graph neural networks,'' in \emph{Advances in
  Neural Information Processing Systems 32}, 2019, pp. 9240--9251.

\bibitem{DBLP:conf/kdd/YuanTHJ20}
H.~Yuan, J.~Tang, X.~Hu, and S.~Ji, ``{XGNN:} towards model-level explanations
  of graph neural networks,'' in \emph{The 26th {ACM} {SIGKDD} Conference on
  Knowledge Discovery and Data Mining}, R.~Gupta, Y.~Liu, J.~Tang, and B.~A.
  Prakash, Eds.\hskip 1em plus 0.5em minus 0.4em\relax {ACM}, 2020, pp.
  430--438.

\bibitem{DBLP:journals/corr/abs-2004-09808}
C.~Ji, R.~Wang, and H.~Wu, ``Perturb more, trap more: Understanding behaviors
  of graph neural networks,'' \emph{CoRR}, vol. abs/2004.09808, 2020.

\bibitem{DBLP:journals/corr/abs-2010-00577}
M.~S. Schlichtkrull, N.~D. Cao, and I.~Titov, ``Interpreting graph neural
  networks for {NLP} with differentiable edge masking,'' \emph{CoRR}, vol.
  abs/2010.00577, 2020.

\bibitem{DBLP:conf/iclr/KipfW17}
T.~N. Kipf and M.~Welling, ``Semi-supervised classification with graph
  convolutional networks,'' in \emph{5th International Conference on Learning
  Representations}.\hskip 1em plus 0.5em minus 0.4em\relax OpenReview.net,
  2017.

\bibitem{schnet_paper}
K.~Sch{\"{u}}tt, P.~Kindermans, H.~E.~S. Felix, S.~Chmiela, A.~Tkatchenko, and
  K.-R. M{\"{u}}ller, ``Sch{N}et: {A} continuous-filter convolutional neural
  network for modeling quantum interactions,'' in \emph{{Neural Information
  Processing Systems}}, 2017, pp. 991--1001.

\bibitem{DBLP:conf/ijcai/YuYZ18}
B.~Yu, H.~Yin, and Z.~Zhu, ``Spatio-temporal graph convolutional networks: {A}
  deep learning framework for traffic forecasting,'' in \emph{Proceedings of
  the 27th International Joint Conference on Artificial Intelligence}, 2018,
  pp. 3634--3640.

\bibitem{hu2020ogb}
W.~Hu, M.~Fey, M.~Zitnik, Y.~Dong, H.~Ren, B.~Liu, M.~Catasta, and J.~Leskovec,
  ``Open graph benchmark: Datasets for machine learning on graphs,''
  \emph{CoRR}, vol. abs/2005.00687, 2020.

\bibitem{DBLP:journals/corr/ShrikumarGSK16}
A.~Shrikumar, P.~Greenside, A.~Shcherbina, and A.~Kundaje, ``Not just a black
  box: Learning important features through propagating activation
  differences,'' \emph{CoRR}, vol. abs/1605.01713, 2016.

\bibitem{DBLP:conf/icml/BalduzziFLLMM17}
D.~Balduzzi, M.~Frean, L.~Leary, J.~P. Lewis, K.~W. Ma, and B.~McWilliams,
  ``The shattered gradients problem: If resnets are the answer, then what is
  the question?'' in \emph{Proceedings of the 34th International Conference on
  Machine Learning}, vol.~70, 2017, pp. 342--350.

\bibitem{DBLP:journals/dsp/MontavonSM18}
G.~Montavon, W.~Samek, and K.~M{\"{u}}ller, ``Methods for interpreting and
  understanding deep neural networks,'' \emph{Digit. Signal Process.}, vol.~73,
  pp. 1--15, 2018.

\bibitem{DBLP:series/lncs/MontavonBLSM19}
G.~Montavon, A.~Binder, S.~Lapuschkin, W.~Samek, and K.-R. M{\"{u}}ller,
  ``Layer-wise relevance propagation: An overview,'' in \emph{Explainable
  {AI}}, ser. Lecture Notes in Computer Science.\hskip 1em plus 0.5em minus
  0.4em\relax Springer, 2019, vol. 11700, pp. 193--209.

\bibitem{DBLP:series/lncs/Montavon19}
G.~Montavon, ``Gradient-based vs. propagation-based explanations: An axiomatic
  comparison,'' in \emph{Explainable {AI}}, ser. Lecture Notes in Computer
  Science.\hskip 1em plus 0.5em minus 0.4em\relax Springer, 2019, vol. 11700,
  pp. 253--265.

\bibitem{spectral_nets_bruna2014}
J.~Bruna, W.~Zaremba, A.~Szlam, and Y.~Lecun, ``Spectral networks and locally
  connected networks on graphs,'' in \emph{2nd International Conference on
  Learning Representations}, 2014.

\bibitem{DBLP:conf/iclr/XuHLJ19}
K.~Xu, W.~Hu, J.~Leskovec, and S.~Jegelka, ``How powerful are graph neural
  networks?'' in \emph{7th International Conference on Learning
  Representations}.\hskip 1em plus 0.5em minus 0.4em\relax OpenReview.net,
  2019.

\bibitem{NIPS2016_6081}
M.~Defferrard, X.~Bresson, and P.~Vandergheynst, ``Convolutional neural
  networks on graphs with fast localized spectral filtering,'' in
  \emph{Advances in Neural Information Processing Systems 29}.\hskip 1em plus
  0.5em minus 0.4em\relax Curran Associates, Inc., 2016, pp. 3844--3852.

\bibitem{DBLP:journals/pr/MontavonLBSM17}
G.~Montavon, S.~Lapuschkin, A.~Binder, W.~Samek, and K.-R. M{\"{u}}ller,
  ``Explaining nonlinear classification decisions with deep {T}aylor
  decomposition,'' \emph{Pattern Recognit.}, vol.~65, pp. 211--222, 2017.

\bibitem{hamilton2017inductive}
W.~L. Hamilton, Z.~Ying, and J.~Leskovec, ``Inductive representation learning
  on large graphs,'' in \emph{Advances in Neural Information Processing Systems
  30}, 2017, pp. 1024--1034.

\bibitem{NIPS2015_5954}
D.~K. Duvenaud, D.~Maclaurin, J.~Iparraguirre, R.~Bombarell, T.~Hirzel,
  A.~Aspuru-Guzik, and R.~P. Adams, ``Convolutional networks on graphs for
  learning molecular fingerprints,'' in \emph{Advances in Neural Information
  Processing Systems 28}.\hskip 1em plus 0.5em minus 0.4em\relax Curran
  Associates, Inc., 2015, pp. 2224--2232.

\bibitem{DBLP:journals/spm/BronsteinBLSV17}
M.~M. Bronstein, J.~Bruna, Y.~LeCun, A.~Szlam, and P.~Vandergheynst,
  ``Geometric deep learning: Going beyond {E}uclidean data,'' \emph{{IEEE}
  Signal Process. Mag.}, vol.~34, no.~4, pp. 18--42, 2017.

\bibitem{DBLP:journals/corr/SimonyanZ14a}
K.~Simonyan and A.~Zisserman, ``Very deep convolutional networks for
  large-scale image recognition,'' in \emph{{3rd International Conference on
  Learning Representations}}, 2015.

\bibitem{DBLP:conf/nips/YingY0RHL18}
Z.~Ying, J.~You, C.~Morris, X.~Ren, W.~L. Hamilton, and J.~Leskovec,
  ``Hierarchical graph representation learning with differentiable pooling,''
  in \emph{Advances in Neural Information Processing Systems 31}, 2018, pp.
  4805--4815.

\bibitem{Albert2002}
R.~Albert and A.-L. Barab{\'{a}}si, ``Statistical mechanics of complex
  networks,'' \emph{Reviews of Modern Physics}, vol.~74, no.~1, pp. 47--97,
  Jan. 2002.

\bibitem{DBLP:journals/tnn/SamekBMLM17}
W.~Samek, A.~Binder, G.~Montavon, S.~Lapuschkin, and K.-R. M{\"{u}}ller,
  ``Evaluating the visualization of what a deep neural network has learned,''
  \emph{{IEEE} Trans. Neural Networks Learn. Syst.}, vol.~28, no.~11, pp.
  2660--2673, 2017.

\bibitem{Jurafsky2009}
D.~Jurafsky and J.~H. Martin, \emph{Speech and language processing: an
  introduction to natural language processing, computational linguistics, and
  speech recognition, 2nd Edition}, ser. Prentice Hall series in artificial
  intelligence, 2009.

\bibitem{rieck2010approximate}
K.~Rieck, T.~Krueger, U.~Brefeld, and K.-R. M{\"u}ller, ``Approximate tree
  kernels.'' \emph{Journal of Machine Learning Research}, vol.~11, no.~16, pp.
  555--580, 2010.

\bibitem{DBLP:series/synthesis/2012Liu}
B.~Liu, \emph{Sentiment Analysis and Opinion Mining}, ser. Synthesis Lectures
  on Human Language Technologies.\hskip 1em plus 0.5em minus 0.4em\relax Morgan
  {\&} Claypool Publishers, 2012.

\bibitem{socher-etal-2013-recursive}
R.~Socher, A.~Perelygin, J.~Wu, J.~Chuang, C.~D. Manning, A.~Y. Ng, and
  C.~Potts, ``Recursive deep models for semantic compositionality over a
  sentiment treebank,'' in \emph{Proceedings of the 2013 Conference on
  Empirical Methods in Natural Language Processing}.\hskip 1em plus 0.5em minus
  0.4em\relax {Association for Computational Linguistic}, 2013, pp. 1631--1642.

\bibitem{NIPS2016_6228}
T.~Bolukbasi, K.~Chang, J.~Y. Zou, V.~Saligrama, and A.~T. Kalai, ``Man is to
  computer programmer as woman is to homemaker? debiasing word embeddings,'' in
  \emph{Advances in Neural Information Processing Systems 29}, 2016, pp.
  4349--4357.

\bibitem{DBLP:journals/corr/abs-1903-03862}
H.~Gonen and Y.~Goldberg, ``Lipstick on a pig: Debiasing methods cover up
  systematic gender biases in word embeddings but do not remove them,'' in
  \emph{Proceedings of the 2019 Conference of the North {A}merican Chapter of
  the Association for Computational Linguistics: Human Language Technologies,
  Volume 1}.\hskip 1em plus 0.5em minus 0.4em\relax Association for
  Computational Linguistics, 2019, pp. 609--614.

\bibitem{von2019exploring}
O.~A. von Lilienfeld, K.-R. M{\"u}ller, and A.~Tkatchenko, ``Exploring chemical
  compound space with quantum-based machine learning,'' \emph{Nat. Rev. Chem.},
  vol.~4, pp. 347--358, 2020.

\bibitem{noe2020machine}
F.~No{\'e}, A.~Tkatchenko, K.-R. M{\"u}ller, and C.~Clementi, ``Machine
  learning for molecular simulation,'' \emph{Annu. Rev. Phys. Chem.}, vol.~71,
  no.~1, pp. 361--390, 2020.

\bibitem{QML-Book}
K.~T. Sch{\"u}tt, S.~Chmiela, O.~A. von Lilienfeld, A.~Tkatchenko, K.~Tsuda,
  and K.-R. M{\"u}ller.\hskip 1em plus 0.5em minus 0.4em\relax Springer Lecture
  Notes in Physics, 2020, vol. 968.

\bibitem{schutt2017quantum}
K.~T. Sch{\"u}tt, F.~Arbabzadah, S.~Chmiela, K.~R. M{\"u}ller, and
  A.~Tkatchenko, ``Quantum-chemical insights from deep tensor neural
  networks,'' \emph{Nature communications}, vol.~8, p. 13890, 2017.

\bibitem{DBLP:conf/icml/GilmerSRVD17}
J.~Gilmer, S.~S. Schoenholz, P.~F. Riley, O.~Vinyals, and G.~E. Dahl, ``Neural
  message passing for quantum chemistry,'' in \emph{{Proceedings of the 34th
  International Conference on Machine Learning }}, vol.~70, 2017, pp.
  1263--1272.

\bibitem{schutt2018schnetpack}
K.~Sch{\"u}tt, P.~Kessel, M.~Gastegger, K.~Nicoli, A.~Tkatchenko, and K.-R.
  M{\"u}ller, ``Schnetpack: A deep learning toolbox for atomistic systems,''
  \emph{Journal of chemical theory and computation}, vol.~15, no.~1, pp.
  448--455, 2018.

\bibitem{ramakrishnan2014quantum}
R.~Ramakrishnan, P.~O. Dral, M.~Rupp, and O.~A. von Lilienfeld, ``Quantum
  chemistry structures and properties of 134 kilo molecules,'' \emph{Scientific
  Data}, vol.~1, no. 40022, 2014.

\bibitem{Schutt2019book}
K.~T. Sch{\"u}tt, M.~Gastegger, A.~Tkatchenko, and K.-R. M{\"u}ller,
  ``Quantum-chemical insights from interpretable atomistic neural networks,''
  in \emph{Explainable {AI:} Interpreting, Explaining and Visualizing Deep
  Learning}, ser. Lecture Notes in Computer Science, W.~Samek, G.~Montavon,
  A.~Vedaldi, L.~K. Hansen, and K.-R. M{\"u}ller, Eds.\hskip 1em plus 0.5em
  minus 0.4em\relax Springer, 2019, vol. 11700, pp. 311--330.

\bibitem{DBLP:conf/eccv/ZeilerF14}
M.~D. Zeiler and R.~Fergus, ``Visualizing and understanding convolutional
  networks,'' in \emph{13th European Conference on Computer Vision}, ser.
  Lecture Notes in Computer Science, vol. 8689.\hskip 1em plus 0.5em minus
  0.4em\relax Springer, 2014, pp. 818--833.

\bibitem{inceptionism}
\BIBentryALTinterwordspacing
A.~Mordvintsev, C.~Olah, and M.~Tyka, ``Inceptionism: Going deeper into neural
  networks,'' 2015. [Online]. Available:
  \url{https://research.googleblog.com/2015/06/inceptionism-going-deeper-into-neural.html}
\BIBentrySTDinterwordspacing

\end{thebibliography}


\begin{thebibliography}{1}
\providecommand{\url}[1]{#1}
\csname url@samestyle\endcsname
\providecommand{\newblock}{\relax}
\providecommand{\bibinfo}[2]{#2}
\providecommand{\BIBentrySTDinterwordspacing}{\spaceskip=0pt\relax}
\providecommand{\BIBentryALTinterwordstretchfactor}{4}
\providecommand{\BIBentryALTinterwordspacing}{\spaceskip=\fontdimen2\font plus
\BIBentryALTinterwordstretchfactor\fontdimen3\font minus
  \fontdimen4\font\relax}
\providecommand{\BIBforeignlanguage}[2]{{%
\expandafter\ifx\csname l@#1\endcsname\relax
\typeout{** WARNING: IEEEtran.bst: No hyphenation pattern has been}%
\typeout{** loaded for the language `#1'. Using the pattern for}%
\typeout{** the default language instead.}%
\else
\language=\csname l@#1\endcsname
\fi
#2}}
\providecommand{\BIBdecl}{\relax}
\BIBdecl

\bibitem{DBLP:series/lncs/MontavonBLSM19}
G.~Montavon, A.~Binder, S.~Lapuschkin, W.~Samek, and K.-R. M{\"{u}}ller,
  ``Layer-wise relevance propagation: An overview,'' in \emph{Explainable
  {AI}}, ser. Lecture Notes in Computer Science.\hskip 1em plus 0.5em minus
  0.4em\relax Springer, 2019, vol. 11700, pp. 193--209.

\bibitem{Albert2002}
R.~Albert and A.-L. Barab{\'{a}}si, ``Statistical mechanics of complex
  networks,'' \emph{Reviews of Modern Physics}, vol.~74, no.~1, pp. 47--97,
  Jan. 2002.

\bibitem{socher-etal-2013-recursive}
R.~Socher, A.~Perelygin, J.~Wu, J.~Chuang, C.~D. Manning, A.~Y. Ng, and
  C.~Potts, ``Recursive deep models for semantic compositionality over a
  sentiment treebank,'' in \emph{Proceedings of the 2013 Conference on
  Empirical Methods in Natural Language Processing}.\hskip 1em plus 0.5em minus
  0.4em\relax {Association for Computational Linguistic}, 2013, pp. 1631--1642.

\bibitem{spacy2}
\BIBentryALTinterwordspacing
M.~Honnibal and I.~Montani, ``{spaCy 2}: Natural language understanding with
  {B}loom embeddings, convolutional neural networks and incremental parsing,''
  2017. [Online]. Available: \url{https://spacy.io/}
\BIBentrySTDinterwordspacing

\bibitem{DBLP:conf/iclr/KipfW17}
T.~N. Kipf and M.~Welling, ``Semi-supervised classification with graph
  convolutional networks,'' in \emph{5th International Conference on Learning
  Representations}.\hskip 1em plus 0.5em minus 0.4em\relax OpenReview.net,
  2017.

\bibitem{kingma2014method}
D.~P. Kingma and J.~Ba, ``Adam: {A} method for stochastic optimization,'' in
  \emph{{International Conference on Learning Representations} (Poster)}, 2015.

\bibitem{schutt2018schnet}
K.~T. Sch{\"u}tt, H.~E. Sauceda, P.-J. Kindermans, A.~Tkatchenko, and K.-R.
  M{\"u}ller, ``Sch{N}et--a deep learning architecture for molecules and
  materials,'' \emph{The Journal of Chemical Physics}, vol. 148, no.~24, p.
  241722, 2018.

\bibitem{schutt2020learning}
K.~T. Sch{\"u}tt, A.~Tkatchenko, and K.-R. M{\"u}ller, ``Learning
  representations of molecules and materials with atomistic neural networks,''
  in \emph{Machine Learning Meets Quantum Physics}.\hskip 1em plus 0.5em minus
  0.4em\relax Springer, 2020, pp. 215--230.

\bibitem{ramakrishnan2014quantum}
R.~Ramakrishnan, P.~O. Dral, M.~Rupp, and O.~A. von Lilienfeld, ``Quantum
  chemistry structures and properties of 134 kilo molecules,'' \emph{Scientific
  Data}, vol.~1, no. 40022, 2014.

\end{thebibliography}

\end{document}


\title{
Higher-Order Explanations of Graph Neural Networks via Relevant Walks\\[2mm]\Large \textsc{(Supplementary Material)}
}

\author{Thomas Schnake, Oliver Eberle, Jonas Lederer, Shinichi Nakajima\\Kristof Sch\"utt, Klaus-Robert M\"uller, Gr\'egoire Montavon}

\maketitle
\allowdisplaybreaks

\appendices

\noindent In this Supplementary Material, we provide the proofs for Propositions \ref{proposition:poshom} and \ref{proposition:nested} of the main paper, on which our method is built. We also provide a detailed justification of the GNN-LRP procedure for the GCN. Finally, we give more details on our synthetic dataset and on the graph neural networks used in the experiments section of the main paper.

\section{Proof of Proposition \ref{proposition:poshom}}

\begin{proposition}
\label{proposition:poshom}
Let $f(\bm{\Lambda})$ have the structure of Eqs.\ (1)--(3) in the main paper, with $\mathcal{C}_t$ and $g$ piecewise linear and positively homogeneous with their respective inputs. If we perform a Taylor expansion of $f(\bm{\Lambda})$ as in Eq. (5) of the main paper, at the reference point $\widetilde{\bm{\Lambda}}= s \bm{\Lambda}$ for an $s > 0$, then all terms of the expansion which are of higher or lower order than the network depth $T$ vanish in the limit of $s \to 0$, and we arrive at a decomposition $f(\bm{\Lambda}) = \sum_{\mathcal{B}} R_{\mathcal{B}}$ with
\begin{align}
R_\mathcal{B} &=  \frac{1}{\alpha_{\mathcal{B}}!}\frac{\partial^T f}{
\partial \lambda_{\mathcal{E}_1} \dots \partial \lambda_{\mathcal{E}_T}
} \cdot
 \lambda_{\mathcal{E}_1} \cdot \hdots \cdot \lambda_{\mathcal{E}_T}
\label{eq:hightaylor-reduced}
\end{align}
\end{proposition}

\begin{proof} 
We know that function $\mathcal{C}_t$ is \textit{piecewise linear}, and from Eq.\ (1) in the main paper we know that the aggregation step in each layer $t$ is linear with respect to $\bm{\Lambda}$. Applying the aggregation and combine step $T$ times with respect to the same $\bm{\Lambda}$, implies that the last hidden layer $\bm{H}_T( \bm{\Lambda} ; \bm{H}_0)$ is piecewise a multivariate polynomial of order $T$ in $\bm{\Lambda}$. Piecewise linearity of the readout function $g$ implies that $f(\bm{\Lambda})$ is \textit{piecewise multivariate polynomial of order $T$} as well. 

We also assumed \textit{positive homogeneity} in the combine and readout function. 
With a $T$-times layer-wise application of the aggregation function with respect to $\bm{\Lambda}$, and the combine function $\mathcal{C}_t$, we get that $\bm{H}_T(\bm{\Lambda}; \bm{H}_0)$ is positive homogeneous of order $T$ in $\bm{\Lambda}$. 
Positive homogeneity of the readout function $g$ implies that $f(\bm{\Lambda})$ is also \textit{positive homogeneous of order $T$}, which means formally
\begin{align}\label{eq:poshom_f}
    f(s \bm{\Lambda}) &= s^T f(\bm{\Lambda}) \qquad \text{for all} \quad s>0.
\end{align}
From the property of positive homogeneity we can deduce two additional important properties. 
First, $f(\bm{\Lambda})$ is not only piecewise a polynomial, but piecewise a homogeneous polynomial of order $T$. This means that for every input matrix $\bm{\Lambda}$, there is a piece $\iota \subset \mathbb{R}^{n \times n}$ with $\bm{\Lambda} \in \iota$ for which we can find a set of coefficients $b_{{\mathcal{B}}} \in \mathbb{R}$, such that the function can be written on that piece as a \textit{homogeneous polynomial of order $T$}:
\begin{align}\label{eq:f_monom}
    f(\bm{\Lambda}) &= \sum_{{\mathcal{B}} \in \mathbb{B}_T } b_{{\mathcal{B}}} \cdot \lambda_{\mathcal{E}_1} \cdot \hdots \cdot \lambda_{\mathcal{E}_T },
\end{align}
where  $\mathbb{B}_T$ is the set of all bag-of-edges of length $T$. Second, the piece $\iota$ can always be chosen large enough so that it contains the whole line $s\bm{\Lambda}$ with $s >0$, i.e. $(s \bm{\Lambda})_{s>0} \subset \iota$, which includes the input $\bm{\Lambda}$ but also the reference point $ \widetilde{\bm{\Lambda}}$. This enables us to restrict our investigation to that particular piece $\iota$, specifically, we restrict in the following the function to be $f:\iota \to \mathbb{R}$ with $f \gets f_{|\iota}$, and leverage the fact that it is a homogeneous polynomial of order $T$.

\smallskip

We now proceed with the main step of Proposition \ref{proposition:poshom} and apply the Taylor theorem. The general Taylor decomposition of $f(\bm{\Lambda})$ at any reference point $\widetilde{\bm{\Lambda}}$ is given by
\begin{align}\label{eq:taylor_f}
    f(\bm{\Lambda}) =& \sum_{k=0}^{\infty} \sum_{\mathcal{B} \in \mathbb{B}_k}
    \frac{1}{\alpha_{\mathcal{B}}!} \frac{\partial^k f}{
\partial \lambda_{\mathcal{E}_1} \dots \partial \lambda_{\mathcal{E}_k}
}\bigg|_{\widetilde{\bm{\Lambda}}}   \cdot \prod_{i=1}^k (\lambda_{\mathcal{E}_i} - \widetilde{\lambda}_{\mathcal{E}_i})
\end{align}
where $\alpha_\mathcal{B}$ is the multindex of the bag $\mathcal{B}$, i.e. $\alpha_{\mathcal{B},\mathcal{E}} \vcentcolon= |\{\mathcal{E} \in \mathcal{B} \}|$  with $\alpha_{\mathcal{B}}! = \prod_{\mathcal{E}} \alpha_{\mathcal{B},\mathcal{E}}! $.
We first want to show that the addends on the r.h.s. with respect to $k$ vanish,  for $k \neq T$, if $\widetilde{\bm{\Lambda}} = s \bm{\Lambda}$ and $s \to 0$. We distinguish between two cases.

\medskip

\begin{description}[itemsep=3mm]
\item [Case $k < T$]
From Eq.\ \eqref{eq:f_monom} we deduce that $ \partial^k f\big/(
\partial \lambda_{\mathcal{E}_1} \dots \partial \lambda_{\mathcal{E}_k})
$ is a homogeneous polynomial of order $T-k$. Which directly implies \begin{align*}\frac{\partial^k f}{ \partial \lambda_{\mathcal{E}_1} \dots \partial \lambda_{\mathcal{E}_k} }\bigg|_{s \bm{\Lambda}} \rightarrow \frac{\partial^k f}{ \partial \lambda_{\mathcal{E}_1} \dots \partial \lambda_{\mathcal{E}_k} }\bigg|_{0} = 0 \qquad \text{for} \quad s \to 0.
\end{align*}

\item[Case $k>T$]
We know that $f(\bm{\Lambda})$ is a homogeneous polynomial of degree $T$. If we consider the derivative of $f$ with a higher order than $T$, it must be zero, hence $\partial^k f \big/
(\partial \lambda_{\mathcal{E}_1} \dots \partial \lambda_{\mathcal{E}_k}) =0 $.
\end{description}

\medskip

\noindent Now, together with the fact that $\partial^k f \big/
 (\partial \lambda_{\mathcal{E}_1} \dots \partial \lambda_{\mathcal{E}_k})$ is constant with respect to $\bm{\Lambda}$ for $k=T$, we obtain for $s \to 0$ 
\begin{align*}
f(\bm{\Lambda}) &= \sum_{\mathcal{B} \in \mathbb{B}_T}
    \frac{1}{\alpha_{\mathcal{B}}!} \frac{\partial^T f}{
\partial \lambda_{\mathcal{E}_1} \dots \partial \lambda_{\mathcal{E}_T}
}  \cdot \prod_{i=1}^T \lambda_{\mathcal{E}_i} 
\end{align*}
which is what we wanted to show.

\end{proof}

\section{Proof of Proposition \ref{proposition:nested}}

\begin{proposition}
For the considered function $f(\bm{\Lambda}^\star)$ the higher-order terms in Eq.\ \eqref{eq:hightaylor-reduced}
can be equivalently computed as a sequence of differentations and multiplications by the terms of $\bm{\Lambda}^\star$ forming each walk $\mathcal{W}=(\dots, J,K,L, \dots)$:
\begin{align}
R_{\mathcal{W}} &= \frac{\partial}{\partial \hdots} \left( \frac{\partial}{\partial \lambda^\star_{JK}} \left( \frac{\partial \hdots}{\partial \lambda^\star_{KL}} \cdot  \lambda_{KL}^\star \right) \cdot \lambda_{JK}^\star \right) \cdot  \hdots
\label{eq:nestedgi}
\end{align}
and then applying the pooling operation $R_{\mathcal{B}} = \sum_{\mathcal{W} \in \mathcal{B}} R_{\mathcal{W}}$.
\label{proposition:nested}
\end{proposition}

\begin{proof}
    First we recall, that with $\bm{\Lambda}^\star \gets (\bm{\Lambda}, \dots, \bm{\Lambda})$ we  distinguish between the connectivity matrices in every interaction block. In our proof we want to index the connectivity matrix in every interaction block and therefore extend the notation to $\bm{\Lambda}^\star = (\lambda_\mathcal{E}^{(t)})_{\mathcal{E},t} $ where $t=1, \dots, T$ identifies the interaction block and $\mathcal{E}$ identifies the edge (or node pair). 
    Then for any given bag-of-edges $\mathcal{B}$ which contains the edges sequence $ \mathcal{E}_1, \dots, \mathcal{E}_T$, we note
    \begin{align}
         R_\mathcal{B}
 & = \frac{1}{\alpha_\mathcal{B} ! }\frac{\partial^T f }{\partial \lambda_{\mathcal{E}_1} \dots \partial \lambda_{\mathcal{E}_T}} \cdot \lambda_{\mathcal{E}_1} \cdot \hdots \cdot \lambda_{\mathcal{E}_T}
 \label{eq:relev_def}\\
        & =  \frac{1}{\alpha_\mathcal{B} ! } \sum_{t_1=1}^T \hdots \sum_{t_T=1}^T \frac{\partial^T f }{\partial \lambda_{\mathcal{E}_1}^{(t_1)} \dots \partial \lambda_{\mathcal{E}_T}^{(t_T)}} \cdot  \lambda_{\mathcal{E}_1}^{(t_1)} \cdot \hdots \cdot \lambda_{\mathcal{E}_T}^{(t_T)}  & \text{(Directional Derivative Expansion)} \label{eq:deriv_transform}\\
        & = \!\!\!\sum_{ \bm{\mathcal{E}} \in \mathcal{O}(\mathcal{B}) }  \frac{\partial^T f }{\partial \lambda_{\mathcal{E}_1}^{(1)} \dots \partial \lambda_{\mathcal{E}_T}^{(T)}} \cdot  \lambda_{\mathcal{E}_1}^{(1)} \cdot \hdots \cdot \lambda_{\mathcal{E}_T}^{(T)} & \text{(Reduction to Unique Layer)} \label{eq:lay_reduc} \\
        & =  \sum_{ \mathcal{W} \in \mathcal{B}}  \frac{\partial^T f }{\partial  \lambda_{\mathcal{E}_1}^{(1)} \dots \partial \lambda_{\mathcal{E}_T}^{(T)}} \cdot  \lambda_{\mathcal{E}_1}^{(1)} \cdot \hdots \cdot \lambda_{\mathcal{E}_T}^{(T)} & \text{(Reduction to Walks)} \label{eq:walk_reduc}
    \end{align}
    
    \medskip
    
    \noindent where $\alpha_\mathcal{B}$ is the multi-index which counts the multiplicity of edges in $\mathcal{B}$, i.e. $\alpha_{\mathcal{B},\mathcal{E}} \vcentcolon= | \{\mathcal{E} \in \mathcal{B} \}| $ for a node pair $\mathcal{E}$, and $\alpha_\mathcal{B}! \vcentcolon= \prod_\mathcal{E} \alpha_{\mathcal{B},\mathcal{E}}!$. The three steps we have performed are detailed below:

    \medskip
    
\begin{description}[itemsep=2mm,leftmargin=1cm]
    \item[Directional Derivative Expansion:]
    For the step between Eq.\ \eqref{eq:relev_def} and \eqref{eq:deriv_transform}, we used that the derivatives of $f$ in one component $\lambda_{\mathcal{E}}$ is equal to the directional derivative of $f$ at the vector $(\lambda_\mathcal{E}^{(t)})_t$ in the direction of the vector which is constant one, i.e. $\bm{1} \in \mathbb{R}^T$. From the theory of directional derivatives we know that the directional derivative is equal to the scalar product between the gradient and its direction, since the direction is in our case just $\bm{1}$, the derivative of $\partial f / \partial \lambda_\mathcal{E}$ is equal to $\sum_t \partial f / \partial \lambda^{(t)}_\mathcal{E}$. Applying this to every edge $\mathcal{E}_t$ in $\mathcal{B}$ leads to Eq.\ \eqref{eq:deriv_transform}.
    
    \item[Reduction to Unique Layer]
    For the step from Eq.\ \eqref{eq:deriv_transform} to \eqref{eq:lay_reduc} we used the fact that $\partial^T f \big/ (\partial \lambda_{\mathcal{E}_1}^{(t_1)} \dots \partial \lambda_{\mathcal{E}_T}^{(t_T)}) $ is only non-zero if all $t_1, \dots, t_T$ differ, because $f(\bm{\Lambda}^\star)$ is piecewise linear in the connectivity matrix $\lambda^{(t)}$ and if we compute the derivative of $f$ with respect to the same interaction block multiple times, it will vanish.
    This reduces the sum onto addends with partial derivatives of the form $  \partial^T f \big/ (\partial \lambda_{\mathcal{E}_1}^{(1)} \dots \partial \lambda_{\mathcal{E}_T}^{(T)}) $, which only depend on ordered edge sequences $\bm{\mathcal{E}} \in \mathcal{O}(\mathcal{B})$. 
    Yet, the partial derivative $ \partial^T f \big/ (\partial \lambda_{\mathcal{E}_1}^{(1)} \dots \partial \lambda_{\mathcal{E}_T}^{(T)})$  for an edge sequence $\bm{\mathcal{E}} \in \mathcal{O}(\mathcal{B})$ occurs exactly $ \alpha_\mathcal{B}!$ times in the sum of Eq.\ \eqref{eq:deriv_transform}, because $\alpha_\mathcal{B}!$ describes the total number of possible permutations of duplicate edges in $\mathcal{B}$. With this, we arrive at Eq.\ \eqref{eq:lay_reduc}.
    
    \item[Reduction to Walks] From Eq.\ \eqref{eq:lay_reduc} to \eqref{eq:walk_reduc} we used the property that the addends with respect to a given edge sequence $\bm{\mathcal{E}}$ in Eq.\ \eqref{eq:lay_reduc} are only non-zero if $\mathcal{W} = (\mathcal{E}_1, \dots, \mathcal{E}_T)$ forms a walk in the graph with adjacency matrix $\bm{\Lambda}$. This comes from the property that all $\lambda_{\mathcal{E}_1}^{(1)}, \hdots , \lambda_{\mathcal{E}_T}^{(T)}$ are only non-zero if all $\lambda_{\mathcal{E}_1}, \dots \lambda_{\mathcal{E}_T}$ are non-zero, since $\lambda^{(t)}$ is always equal to $\bm{\Lambda}$ for each $t$. The latter condition can also be interpreted as $\mathcal{E}_1, \dots, \mathcal{E}_T$ forming a walk on $\bm{\Lambda}$.
\end{description}
    
    \medskip
    
\noindent To pursue the derivation, we will slightly change the notation, and represent the walks in the upcoming equations as a sequence of nodes rather than edges, i.e. $\mathcal{W}=( \hdots, J, K,L, \hdots)$. With this we get
    \begin{align}
        R_\mathcal{B}& =  \sum_{ \mathcal{W} \in \mathcal{B}}  \frac{\partial^T f }{\partial  \lambda_{\mathcal{E}_1}^{(1)} \dots \partial \lambda_{\mathcal{E}_T}^{(T)}} \cdot  \lambda_{\mathcal{E}_1}^{(1)} \cdot \hdots \cdot \lambda_{\mathcal{E}_T}^{(T)} \label{eq:step2_1}\\
        &= \sum_{\mathcal{W} \in \mathcal{B}}  \frac{\partial^T f }{\dots \partial  \lambda_{JK}^{(t)} \ \partial \lambda_{KL}^{(t+1)} \dots} \cdot \hdots \cdot \lambda_{JK}^{(t)} \cdot \lambda_{KL}^{(t+1)} \cdot \hdots & \text{(Use Node Notation)}\label{eq:edge_rewr}\\
        &= \sum_{\mathcal{W} \in \mathcal{B}}\frac{\partial }{\partial \dots }\left( \frac{\partial }{\partial \lambda_{JK}^{(t)} } \left( \frac{\partial \dots }{\partial \lambda_{KL}^{(t+1)} } \cdot \lambda_{KL}^{(t+1)} \right) \lambda_{JK}^{(t)} \right) \cdot \dots & \text{(Derivative Reordering)}\label{eq:deriv_reord}\\
        &= \sum_{\mathcal{W} \in \mathcal{B}}\frac{\partial}{\partial \dots }\left( \frac{\partial }{\partial \lambda_{JK}^\star } \left( \frac{\partial \dots }{\partial \lambda_{KL}^\star } \cdot \lambda_{KL}^\star \right) \lambda_{JK}^\star \right) \cdot \dots & \text{(Use Implicit Layer Notation)}\label{eq:deriv_rewr}
    \end{align}

\medskip

\noindent Together, with the definition of $R_\mathcal{W}$ in Eq.\ \eqref{eq:nestedgi}, we arrive at $R_\mathcal{B} = \sum_{\mathcal{W} \in \mathcal{B}} R_\mathcal{W}$, which finishes the proof. The last three steps we have performed are discussed in more details below:

\medskip

\begin{description}[itemsep=2mm,leftmargin=1cm]
    \item[Use Node Notation]  Eq.\ \eqref{eq:edge_rewr} is equal to Eq.\ \eqref{eq:step2_1}, by just incorporating  the walk notation $\mathcal{W} = ( \hdots, J, K, L, \hdots )$.
    
    \item[Derivative Reordering]
    In the step from Eq.\ \eqref{eq:edge_rewr} to \eqref{eq:deriv_reord}, we explicitly reorder the differentiation operations and the scalar multiplication from the last layer to the first layer. This is possible because the parameter which fall in this order inside of the brackets are independent of the outer derivatives. Specifically, because parameters at each layer are distinct, a multiplication by some parameter at layer $t$ can always be permuted with the differentiation w.r.t.\ a parameter at some layer $t' \neq t$.
    
    \item[Use Implicit Layer Notation]
    The step from Eq.\ \eqref{eq:deriv_reord} to \eqref{eq:deriv_rewr} introduces a more implicit writing of the derivatives by replacing $\partial / \partial \lambda_{JK}^{(t)}$ with $\partial / \partial \lambda_{JK}^\star$, where the layer index can always be inferred from the nodes variables.
\end{description}
    
\end{proof}

\section{Deriving and Justifying the GNN-LRP Procedure}
\label{section:theory}

In this section, we derive the LRP rule for the GCN shown in Table 1 of the main paper, and provide a justification for the overall GNN-LRP procedure in this context. For this, we follow the three steps of the induction outlined in Section 3.2 of the main paper.

\medskip

We start with \textbf{Step 1} which states the assumption that in layer $t$ the relevance can be written as a dot product
$$
R_{kL\dots} = h_k c_{kL\dots}
$$
where $h_k$ is the activation of neuron $k$, and $c_{kL\dots}$ is a scalar that is locally approximately constant. Here, we mean constancy w.r.t.\ lower-layer activations. 

\medskip

We now continue with \textbf{Step 2} and define the map $((\lambda_{jk})_{j \in J})_J \to R_{kL\dots}$ and where we treat $c_{kL\dots}$ to be constant. In the case of the GCN, the map can be organized into the following two-layer structure:
\begin{align*}
h_k &=  \rho\Big(\sum_J \sum_{j \in J} \lambda_{jk} v_{jk}\Big) & \text{(layer 1)} \\
R_{kL\dots} &= ~~~~h_k c_{kL\dots} & \text{(layer 2)}
\end{align*}
In the equations above, we have used the additional variable $v_{jk} = h_{j} w_{jk}$ and we have broken down connectivity of the input graph to the neuron level, i.e.\ $\lambda_{jk} = \lambda_{JK}$ for each $j \in J$ and $k \in K$. Layer $1$ is a Linear-ReLU layer receiving positive values $\lambda_{jk}$ as input. Various LRP rules can be used in this layer, for example, LRP-$\gamma$, or LRP-$0$. For this layer, we decide to apply LRP-$\gamma$ due to its added robustness \cite{DBLP:series/lncs/MontavonBLSM19}. Note that LRP-$\gamma$ covers LRP-$0$ as a special case with $\gamma=0$. Because the second layer is a simple scalar multiplication, the LRP propagation is trivial in that layer, and we can therefore start directly with the task of redistributing $R_{kL\dots}$ through layer 1. Application of the rule LRP-$\gamma$ in this layer gives:
\begin{align*}
R_{jKL\dots} &= \sum_{k \in K} \frac{\lambda_{jk} \cdot (v_{jk} + \gamma \max(0,v_{jk}))}{\sum_J \sum_{j \in J} \lambda_{jk} \cdot (v_{jk} + \gamma \max(0,v_{jk}))} R_{kL\dots}
\end{align*}
Note that unlike standard LRP, pooling is only performed over a subset of neurons from the higher-layer, specifically, neurons belonging to node $K$. This difference is explained by the fact that the connections $\bm{\lambda}_{JK}$, which we consider as input in our GNN-LRP procedure, only reach a particular node $K$, and consequently, they do not receive relevance from other neurons $K'$ in the same layer. Observing that $v_{jk} + \gamma \max(0,v_{jk}) = h_j \cdot (w_{jk} + \gamma \cdot \max(0,w_{jk}))$ due to $h_j \geq 0$, we obtain the propagation rule:
\begin{align}
R_{jKL\dots} &= \sum_{k \in K}\frac{\lambda_{JK} h_{j} w_{jk}^\wedge}{\sum_J \sum_{j \in J} \lambda_{JK} h_{j} w_{jk}^\wedge } R_{kL\dots}
\label{eq:gcn}
\end{align}
where we have made use of the notation $(\cdot)^\wedge = (\cdot) + \gamma \max(0,\cdot)$. This propagation rule is exactly the same rule as \mbox{Eq.\ (9)} of the main paper.

\medskip

We now apply \textbf{Step 3}, which consists of verifying that quantity produced by the LRP rule has a similar product structure to the one we have assumed in Step 1. For this, we observe that the LRP rule can be rewritten as:
\begin{align}
R_{jKL\dots} &=  h_{j} \underbrace{\sum_{k \in K}  \lambda_{JK} w_{jk}^\wedge\frac{\rho(\sum_J \sum_{j \in J} \lambda_{JK} h_{j} w_{jk})}{\sum_J \sum_{j \in J} \lambda_{JK} h_{j} w_{jk}^\wedge } c_{kL\dots}}_{c_{jKL\dots}}
\label{eq:gcn-reorder}
\end{align}
From there, we observe that the term $c_{jKL\dots}$ depends on activations $(h_j)_j$ only through multiple nested sums, or via the term $c_{kL\dots}$ that we have however assumed to be locally approximately constant in Step 1. This weak dependence of $c_{jKL\dots}$ on activations provides a justification for performing the same locally constant approximation in this layer. 

\smallskip

Overall, the three steps that we have applied verify the inductive principle on which GNN-LRP is based. Consequently, this justifies an application of Eq.\ \eqref{eq:gcn} (i.e.\ Eq.\ (9) in the main paper) from the output of the GNN to the first layer, in order to compute relevant walks.

\smallskip

Note that the product approximation does not hold for every rule. For example, if we would have defined the map $(\lambda_{JK})_J \to R_{kL\dots}$ directly, without breaking down graph connectivity to the neuron level, we would have obtained for the first layer the alternate form $h_k =  \rho(\sum_J \lambda_{JK} \sum_{j \in J} v_{jk})$, and application of LRP-$\gamma$ would have given us the alternate propagation rule:
$$
R_{JKL\dots} = \sum_{k \in K}\frac{\lambda_{JK} (\sum_{j \in J} h_{j} w_{jk})^\wedge}{\sum_J \lambda_{JK} (\sum_{j \in J} h_{j} w_{jk})^\wedge } R_{kL\dots}
$$
However, from this alternate rule, it is not possible to extract the same product structure as we had before. Consequently, within our framework, the propagation process cannot be continued further.

\smallskip

While we have demonstrated our inductive principle on the GCN, it can also be applied in a similar fashion to other models, for example, the GIN and the spectral network given in Table 1 of the main paper. Consequently, we can also obtain a justification for applying GNN-LRP to these models.

\section{Parallel Computation of Relevant Walks for Locally Connected Graphs}

Graphs with local connectivity are commonly encountered in real-world applications (e.g.\ neighborhood graphs, or lattices). In such graphs, GNN-LRP can be strongly accelerated if we can identify selections of nodes $(K)_{K \in \mathcal{K}}$ at a given layer, such that their receptive fields (RFs) in the layer below are disjoint. If that is the case, relevance scores for multiple walks can be computed in parallel by application of a single backward pass. The approach is illustrated in Fig.\ \ref{fig:vgg_walk_detail}.

\begin{figure}[h!]
	\centering
	\includegraphics[width=.9\linewidth]{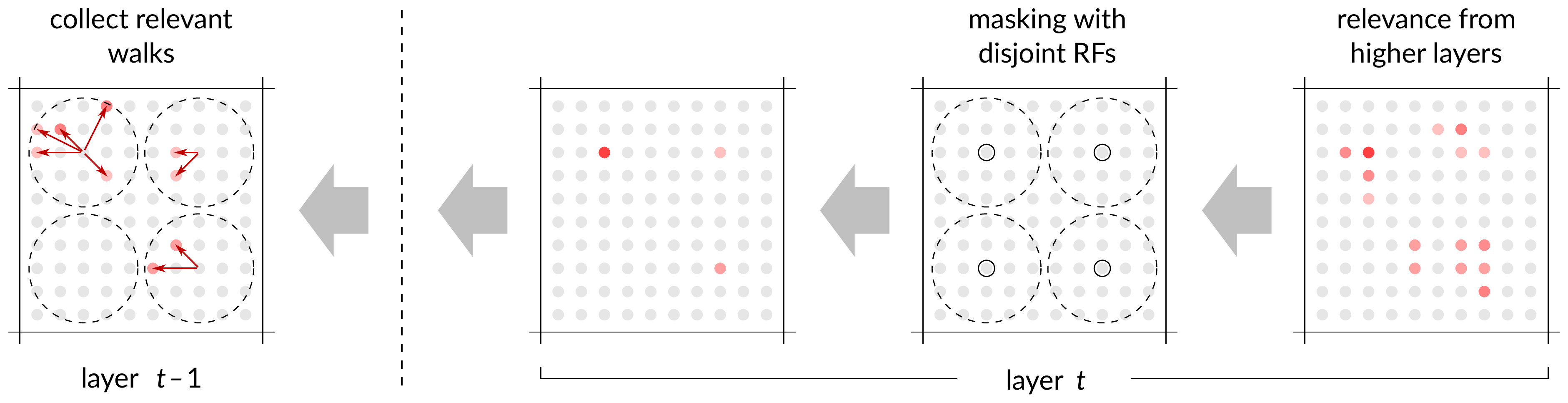}\vskip -2mm
	\caption{Diagram of our method for efficient computation of relevant walks when the input graph has local connectivity, here, a lattice. Red circles represent relevant nodes, and red arrows represent relevant walks from one layer to the previous layer.}
	\label{fig:vgg_walk_detail}
\end{figure}

This approach can be easily implemented by first recalling the implementation trick described in the main paper that consists of rewriting the GCN combine step as: 
\begin{align*} 
\bm{P}_t &\gets \bm{Z}_t \bm{W}_t^\wedge\\
\bm{Q}_t &\gets \bm{P}_t \odot [\kern1pt \rho(\bm{Z}_t \bm{W}_t) \oslash  \bm{P}_t ]_\text{cst.}\\[1mm]
\bm{H}_t &\gets \bm{Q}_t \odot \bm{M}_K + [\bm{Q}_t]_\text{cst.} \odot (\bm{1}-\bm{M}_K),
\end{align*}
where the third line selects the node $K$ that the walk traverses at layer $t$. Our approach consists of replacing the mask $\bm{M}_K$ by another mask $\bm{M}_\mathcal{K} = \sum_{K \in \mathcal{K}} \bm{M}_K$ where $\mathcal{K}$ represents a set of nodes with disjoint RFs. If we are interested in walks from layer $t-1$ to layer $t$, we first run the forward pass with the new mask $\bm{M}_\mathcal{K}$ set at layer $t$. We then apply $\bm{H}_{t-1} \odot \textsc{Autograd}(f,\bm{H}_{t-1})$ which produces an attribution on nodes at layer $t-1$. Finally, we interpret the attribution as the relevance of walks from each node $J$ going into the node $K$ that contains $J$ in its receptive field. To obtain all walks between the two layers, the process must be repeated with different masks $\bm{M}_\mathcal{K}$ until all nodes $K$ at layer $t$ have been covered.

\section{Details on the Synthetic Dataset}

The synthetic dataset we use in our experiments consists of graphs of 20 nodes, generated from two different classes. The first class is Barab\'asi-Albert graphs \cite{Albert2002} with a growth parameter $1$. That is, we start with a graph of two connected nodes, and at each step, we add an additional node and connect it to a node $\mathcal{V}$ from the current graph $\mathcal{G}$ randomly from the distribution
$$
p(\mathcal{V}) =  \frac{\text{degree}(\mathcal{V})}{\sum_{\mathcal{V}' \in \mathcal{G}}\text{degree}(\mathcal{V}')}
$$
The second class has a slightly higher growth model, where the 5th, 10th, 15th and 20th added node are connected to two nodes from the current graph instead of one. For the second class, nodes are selected without replacement with the probability
$$
p(\mathcal{V}) =  \frac{\text{degree}(\mathcal{V})^{-1}}{\sum_{\mathcal{V}' \in \mathcal{G}}\text{degree}(\mathcal{V}')^{-1}},
$$
i.e.\ an inverse preferential attachment model. Compared to the Barab\'asi-Albert graphs, the second class consists of graphs where connections are more evenly distributed between nodes. Because the dataset is generated in a synthetic manner, we can generate arbitrary large datasets, and models trained on such data are therefore not subject to overfitting.

\section{Details of the GNNs and their Training}

In the following, we give details on the design, training, and preparation of the GNNs used in the paper, in particular, the GNNs trained on our Synthetic Data, the GNN trained on the Sentiment Analysis task, the 3-Layer SchNet for quantum chemistry, and the VGG-16 convolutional neural network.

\subsection{GNNs trained on Synthetic Data}

The GNNs trained on synthetic data has $2$ output neurons representing each class and is trained using the binary cross-entropy loss. This ensures that the output zero of the model corresponds to a probability score of $0.5$.

We consider a GCN, a GIN, and a ChebNet. Each network is composed of two interaction layers, followed by a global average pooling as a readout function. Nodes in each layer are represented with $128$ neurons for the GCN, and $32$ neurons for the GIN and ChebNet. The combine function of the GIN is itself a two-layer network. The GCN and the GIN receive as input $\bm{\Lambda} = \tilde{\bm{A}} / 2$, and the ChebNet receives the power expansion $\bm{\Lambda} = [\tilde{\bm{A}}^0,\frac12 \tilde{\bm{A}}^1,\frac14 \tilde{\bm{A}}^2]$, where $\tilde{\bm{A}}$ denotes the adjacency matrix of the undirected input graph to which we have added the self-connections.

Because nodes in our synthetic dataset do not possess information by themselves, we simply encode them in the initial state $\bm{H}_0$ as a one-dimensional vector with value $1$. As a result, the first layer essentially extracts the degree of each node, and the second layer aggregates degrees from adjacent layers.

To avoid the presence of unexplainable factors, we force biases to be non-positive. We implement this by reparameterizing the biases via the scaled softmin function $b = -0.5 \log (1+ \exp(-2 b_0))$ and optimizing $b_0$ instead of $b$. The networks are trained using stochastic gradient descent with momentum $0.9$ and decreasing learning rate.

\subsection{GNNs trained on Sentiment D-Tree}
The GNN which predicts the sentiment of text is trained on the Stanford Sentiment Treebank (SST) \cite{socher-etal-2013-recursive}. In the preprocessing of the corpus, we reduce the number of 5 classes into 2, by neglecting all sentences which are of neutral sentiment and combine the two positive and two negative classes to one positive and one negative class respectively. 
We further generate for each sentence the dependency tree from the spaCy model \texttt{en\_core\_web\_sm} \cite{spacy2}. We do not use the consecutive parse tree provided by the SST corpus, because we saw that the interpretation of consecutive parse trees are less informative than the interpretation of dependency parse trees. We train and test the model on the data split proposed by the author.

The initial state $\bm{H}_0$ is built as follows: Consider the sample $(\mathcal{G}, l)$, with graph $\mathcal{G} = ( \bm{A}, N)$, where $\bm{A}$ is the adjacency matrix of the dependency tree and $N$ are the words of the text, and $l$ is the sentiment label of the graph. To find an intial representation of the sentence, given by $N$, we take vector representations of a pretrained word embedding $h_w$, a randomly initialized word embedding $h_v$, an embedding for the part-of-speech $h_p$ and an embedding for the stemmed words $h_l$. We set the network initialization to be $\bm{H}_0 = [ h_v, h_w, h_p, h_l]$, where we keep $h_v$ fixed during training and $h_w, h_p$ and $h_l$ learnable.

In the forward propagation we first apply a feed-forward neural network (FFN) with ReLU activation simultaneously on each embedded word in $\bm{H}_0$, to obtain a hidden representation of dimension $d_h$ of each word.
It follows the $T$-layer interaction unit, composed of GCNs \cite{DBLP:conf/iclr/KipfW17} with  connectivity matrix $\bm{\Lambda} = \tilde{\bm{D}}^{-\frac{1}{2}}\tilde{\bm{A}}\tilde{\bm{D}}^{-\frac{1}{2}}$,  $\tilde{\bm{D}} = \text{diag}((\sum_{i} \tilde{\bm{A}}_{ij})_j)$ and $\tilde{\bm{A}}$ is the  adjacency matrix of the undirected input graph to which we have added the self-connections. 
Before the readout, we apply again a FFN with ReLU activation. The readout of the provided hidden representation is done by a global average pooling in node direction, to obtain a common vector of dimension $d_h$, followed by a linear layer onto the target dimension and a softmax.

For the training we use the cross-entropy loss between the output and label $l$, and the Adam optimizer \cite{kingma2014method} with learning rate 0.0002. For the results in the main paper we used a model with a hidden dimension $d_h=10$ and the number of layer $T=3$ because it turned out to give the best performance. We obtain on the test split an accuracy of $ \sim 77\%$.

\subsection{3-Layer SchNet}

The SchNet \cite{schutt2018schnet} is a graph neural network used to predict quantum chemical properties of molecular graphs. Each node of the input graph represents an atom of the molecule with its atomic number, and edges of the graph represent the distance between two atoms. The SchNet is organized as a sequence of interaction blocks, where each of them is a composition of an aggregation and combine step. The \textit{aggregation step} consists of a skip and interaction computation:
\begin{align*}
    \bm{Z}_t^\text{Id} &= \bm{\Lambda}_t^\text{Id} \bm{H}_{t-1}\\
    \bm{Z}_t &= \big(\textstyle \sum_{J \in \mathcal{N}(K)} \text{MLP}_\text{CF}^{(t)}(\bm{e}_{JK} )\odot \phi_t \bm{H}_{t-1,J}\big)_{K}
\end{align*}
where $\mathcal{N}(K)$ denotes the neighbors of node $K$, where $\text{MLP}^{(t)}_\text{CF}$ is a neural network mapping the edge feature to a continuous convolution filter~\cite{schutt2020learning}, where ``$\odot$'' denotes the element-wise product, and where $\phi_t$ is a linear transformation.
The edge features are given by the radial basis expansion
$
\bm{e}_{JK} = \left[ \exp (-\delta (D_{JK} -  \mu )^2 )\right]_{\mu \in M}
$
where $M$ is a uniform grid between 0 and the cutoff distance $\mu_{\max}$ with grid width $\Delta \mu$ and resolution $\delta = \frac{1}{2 \Delta \mu^2}$. 
The scalar values $D_{JK}$ define the distance between the atoms $J$ and $K$. This is followed by the \textit{combine step}:
\begin{equation}\label{eq:schnet_trans}
    \bm{H}_t = 
        \bm{Z}_t^\text{Id} + 
        (\text{MLP}^{(t)}(\bm{Z}_{t,K}))_{K}~.
\end{equation}
The skip connection in Eq.\ \eqref{eq:schnet_trans} is comparable to a self-loop in the molecular graph. In our experiments, we use three interaction blocks, a feature dimension of 128, and a grid resolution of $\delta = 9.7 \  \mathring{A} ^{-2}$. We choose a relatively short cutoff distance of $\mu_{\max} = 2.5\ \mathring{A}$ which leads to a sparser connectivity. The SchNet is trained on a subset of the QM9 dataset~\cite{ramakrishnan2014quantum} involving 110,000 data points. The trained models predict atomization energy and dipole moment within a mean absolute error of 0.015 eV and 0.039 Debye, respectively.

\subsection{VGG-16 Network}

We use the pretrained VGG-16 neural network (without batch normalization), provided as part of the PyTorch library and do not change its structure nor retrain it.

\bibliographystyle{IEEEtran}
\bibliography{gnn}